\documentclass[11pt]{article}
\usepackage{fullpage}
\usepackage{amsfonts}
\usepackage{amssymb,amsmath,amsthm}
\usepackage{graphicx}
\usepackage{algorithm}
\usepackage{algpseudocode}
\usepackage{mathpazo,bbm}

\AtBeginDocument{
\DeclareSymbolFont{AMSb}{U}{msb}{m}{n}
    \DeclareSymbolFontAlphabet{\mathbb}{AMSb}}

\newtheorem{theorem}{Theorem}[section]

\newtheorem{lemma}[theorem]{Lemma}

\newtheorem{defn}[theorem]{Definition}

\newtheorem{proposition}[theorem]{Proposition}
\newtheorem{corollary}[theorem]{Corollary}
\newcommand{\set}[1]{\ensuremath \{#1\}}

\renewcommand{\epsilon}{\varepsilon}

\newcommand{\1}{\mbox{\textbf{1}}}

\newcommand{\E}{{\mathbf E}}
\renewcommand{\Pr}{{\mathbf{Pr}}}

\renewcommand{\le}{\leqslant}
\renewcommand{\ge}{\geqslant}
\renewcommand{\leq}{\leqslant}
\renewcommand{\geq}{\geqslant}

\usepackage{color}

\newcommand{\High}{\mathrm{High}}

\begin{document}
\title{Time-Space Tradeoffs for Learning from Small Test Spaces: Learning Low
Degree Polynomial Functions} 

\author{Paul Beame\thanks{Research supported in part by NSF grant CCF-1524246} \\ University of Washington \\ beame@cs.washington.edu \and Shayan Oveis Gharan\thanks{Research supported in part by NSF grant CCF-1552097 and ONR-YI grant N00014-17-1-2429}\\ University of Washington \\ shayan@cs.washington.edu \and Xin Yang$^*$ \\University of Washington\\yx1992@cs.washington.edu}

\date{\today}
\maketitle
\begin{abstract}  
We develop an extension of recently developed methods for obtaining time-space
tradeoff lower bounds for problems of learning from random test samples to
handle the situation where the space of tests is signficantly smaller
than the space of inputs, a class of learning problems that is not handled by
prior work.
This extension is based on a measure of how matrices amplify the 
2-norms of probability distributions that is more refined than the 
2-norms of these matrices.

As applications that follow from our new technique,
we show that 
any algorithm that learns $m$-variate homogeneous polynomial
functions of degree at most $d$ over $\mathbb{F}_2$ from evaluations on
randomly chosen inputs either requires space
$\Omega(mn)$ or $2^{\Omega(m)}$ time where $n=m^{\Theta(d)}$ is the dimension
of the space of such functions.   These bounds are asymptotically optimal since
they match the tradeoffs achieved by natural learning algorithms for the
problems.
\end{abstract}
\newpage
\section{Introduction}

The question of how efficiently one can learn from random samples is a problem
of longstanding interest.   Much of this research has been focussed on the
number of samples required to obtain good approximations.   However, another
important parameter is how much of these samples need to be
kept in memory in order to learn successfully.
There has been a line of work improving the memory efficiency of
learning algorithms, and the question of the limits of such improvement has
begun to be tackled relatively recently.
Shamir~\cite{DBLP:conf/nips/Shamir14} and 
Steinhardt, Valiant, and Wager~\cite{DBLP:conf/colt/SteinhardtVW16} both
obtained constraints on the space required for certain learning problems and
in the latter paper, the authors  asked whether one could obtain strong
tradeoffs for learning from random samples that yields a superlinear threshold
for the space required for efficient learning.
In a breakthrough result, Ran Raz~\cite{DBLP:conf/focs/Raz16} showed that
even given exact information, if the space of a learning algorithm is
bounded by a sufficiently small quadratic function of the input size, then the
parity learning problem given exact answers on random samples requires an
exponential number of samples even to learn an unknown parity function approximately.

More precisely, in the problem of parity learning, an unknown $x\in \{0,1\}^n$  is chosen uniformly at random,
and a learner tries to learn $x$ from a stream of samples $(a^{(1)},b^{(1)},(a^{(2)},b^{(2)}),\cdots$ where $a^{(t)}$ is chosen uniformly at random from $\{0,1\}^n$ and $b^{(t)}=a^{(t)}\cdot x\pmod 2$.
With high probability $n+1$ uniformly random samples suffice to span
$\{0,1\}^n$ and one can solve parity learning using
Gaussian elimination with $(n+1)^2$ space.
Alternatively, an algorithm with only $O(n)$ space can wait for a specific
basis of vectors $a$ to appear (for example the standard basis) and store the 
resulting values; however, this takes $O(2^n)$ time.
Ran Raz \cite{DBLP:conf/focs/Raz16} showed that either $\Omega(n^2)$ space 
or $2^{\Omega(n)}$ time is essential: even if the space is bounded by $n^2/25$,
$2^{\Omega(n)}$ queries are required to learn $x$ correctly with any
probability that is $2^{-o(n)}$.
In follow-on work, \cite{DBLP:journals/eccc/KolRT16} showed that the same
lower bound applies even if the input $x$ is sparse.

We can view $x$ as a (homogeneous) linear function over $\mathbb{F}_2$,
and, from this perspective,
parity learning learns a linear Boolean function from evaluations over
uniformly random inputs.
A natural generalization asks if a similar lower bound exists when we learn
higher order polynomials with bounded space. 

For example, consider homogenous quadratic functions over $\mathbb{F}_2$.
Let $n=\binom{m+1}{2}$ and $X=\{0,1\}^n$, which we identify with the space of
quadratic polynomials in $\mathbb{F}_2[z_1,\ldots,z_m]$
or, equivalently, the space of upper triangular Boolean matrices.
Given an input $x\in \{0,1\}^n$, the learning algorithm receives a stream
of sample pairs $(a^{(1)},b^{(1)}),(a^{(2)},b^{(2)}),\ldots$ where $b^{(t)}=x(a^{(t)})$ 
(or equivalently $b^{(t)}=(a^{(t)})^T x a^{(t)}$ when $x$ is viewed as a
matrix).
A learner tries to learn $x\in X$ with a stream of samples $(a^{(1)},b^{(1)}),(a^2,b^2),\cdots$ where $a^t$ is chosen uniformly at random from $\{0,1\}^m$ and
$b^{(t)}=x(a^{(t)}):=\sum_{i\leq j} x_{ij}a^{(t)}_ia^{(t)}_j\bmod 2$.

Given $a\in \{0,1\}^m$ and $x\in \{0,1\}^n$, we can also view evaluating $x(a)$
as computing $aa^T\cdot x\bmod 2$ where we can interpret $aa^T$ as an element
of $\{0,1\}^n$.
For $O(n)$ randomly chosen $a\in \{0,1\}^m$, the vectors $aa^T$ almost
surely span $\{0,1\}^n$ and hence we only need to store $O(n)$ samples of
the form $(a,b)$ and apply Gaussian elimination to determine $x$.
This time, we only need $m+1$ bits to store each sample for a total space
bound of $O(mn)$.  
An alternative algorithm using $O(n)$ space and time $2^{O(m)}$ would be to look
for a specific basis.   One natural example is the basis consisting of the
upper triangular parts of $$\{e_i e_i^T\mid 1\le i\le m\}\cup
\{(e_i+e_j)(e_i+e_j)^T\mid 1\le i<j \le m\}.$$
We show that this tradeoff between $\Omega(mn)$ space or $2^{\Omega(m)}$ time
is inherently required to learn $x$ with probability $2^{-o(m)}$.

Another view of the problem of learning homogenous quadratic functions (or
indeed any low degree polynomial learning problem) is to
consider it as parity learning with a smaller sample space of tests. 
That is,
we still want to learn $x\in \{0,1\}^n$ with samples $\{(a^{(t)},b^{(t)})\}_{t}$
such that $b^{(t)}= a^{(t)}\cdot x\bmod 2$,
but now $a^{(t)}$ is not chosen uniformly at random from $\{0,1\}^n$; 
instead, we choose $c^{(t)}\in \{0,1\}^m$ uniformly at random and set $a^{(t)}$
to be the upper triangular part of $c^{(t)} (c^{(t)})^T$.
Then the size of the space $A$ of tests is $2^{m}$ which is $2^{O(\sqrt{n})}$
and hence is much smaller than the size $2^n$ space $X$.

Note that this is the dual problem to that considered by
\cite{DBLP:journals/eccc/KolRT16} whose lower bound applied when the unknown
$x$ is sparse, and the tests $a^{(t)}$ are sampled from the whole space.
That is, the space $X$ of possible inputs is much smaller than the space
$A$ of possible tests.

The techniques in \cite{DBLP:conf/focs/Raz16,DBLP:journals/eccc/KolRT16}
were based on fairly ad-hoc simulations of the original space-bounded
learning algorithm by a restricted form of linear branching program
for which one can measure progress at learning $x$ using the dimension of the
consistent subspace. 
More recent papers of Moshkovitz and Moshkovitz~\cite{DBLP:journals/eccc/MoshkovitzM17,MoshkovitzM17new}
and Raz~\cite{DBLP:journals/eccc/Raz17} consider more general 
tests and use a measure of progress based on 2-norms.
While the method of \cite{DBLP:journals/eccc/MoshkovitzM17} is not strong
enough to reproduce the bound in \cite{DBLP:conf/focs/Raz16} for the
case of parity learning, the methods of \cite{DBLP:journals/eccc/Raz17} and
later~\cite{MoshkovitzM17new} 
reproduce the parity learning bound and more. 

In particular, \cite{DBLP:journals/eccc/Raz17} considers an arbitrary space
of inputs $X$ and an arbitrary sample space of tests $A$ and defines a $\pm 1$
matrix $M$ that is indexed by $A\times X$ and has distinct columns;
$M$ indicates the outcome of applying the test $a\in A$ to the input $x\in X$.
The bound is governed by the (expectation) matrix norm of $M$, which is 
is a function of the largest singular value of $M$, and the progress is analyzed
by bounding the impact of applying the matrix to probability distributions with
small expectation $2$-norm.
This method works fine if $|A|\ge |X|$ - i.e., the space of tests is at least as
large as the space of inputs - but it fails completely if $|A|\ll |X|$ which
is precisely the situation for learning quadratic functions.
Indeed, none of the prior approaches work in this case.

In our work we define a property of matrices $M$ that allows us to refine the
notion of the largest singular value and extend the method
of~\cite{DBLP:journals/eccc/Raz17} to the
cases that $|A|\ll |X|$ and, in particular, to prove time-space tradeoff
lower bounds for learning homogeneous quadratic functions over $\mathbb{F}_2$.
This property, which we call the
{\em norm amplification curve} of the matrix on the positive orthant,
analyzes more precisely
how $\|M\cdot p\|_2$ grows as a function of $\|p\|_2$ for probability
vectors $p$ on $X$. The key reason that this is not simply governed by the
singular values is that such $p$ not only have fixed $\ell_1$ norm, they are
also on the positive orthant, which can contain at most one singular vector.
We give a simple condition on the 2-norm amplification curve of $M$ that
is sufficient to ensure that there is a time-space tradeoff showing
that any learning algorithm for $M$ with success probability at least $2^{-\varepsilon m}$ for some $\varepsilon>0$ either requires space $\Omega(nm)$ or
time $2^{\Omega(m)}$.

For any fixed learning problem given by a matrix $M$, the natural way to
express the amplification curve at any particular value of
$\|p\|_2$ yields an optimization problem given by a quadratic program with
constraints on $\|p\|^2_2$, $\|p\|_1$ and $p\ge 0$, and with objective function
$\|Mp\|^2_2=\langle M^TM,pp^T\rangle$ that seems difficult to solve.
Instead, we relax the quadratic program to a semi-definite program where we
replace $pp^T$ by a positive semidefinite matrix $U$ with the analogous
constraints.   We can then obtain an upper
bound on the amplification curve by moving to the SDP dual and evaluating the
dual objective at a particular Laplacian determined by the properties of
$M^T M$.

For matrices $M$ associated with low degree polynomials over $\mathbb{F}_2$,
the property of the matrix $M^T M$ required to bound the amplication
curves for $M$ correspond precisely to properties of the weight distribution of
Reed-Muller codes over $\mathbb{F}_2$.
In the case of quadratic polynomials, we can analyze this weight distribution
exactly.  In the case of higher degree polynomials, bounds on the weight
distribution of such codes proven by
Kaufman, Lovett, and Porat~\cite{DBLP:journals/tit/KaufmanLP12} are
sufficient to obtain the properties we need to give strong enough bounds on the 
norm amplification curves to yield the time-space tradeoffs for learning for
all degrees $d$ that are $o(\sqrt{m})$.

Our new method extends the potential reach of time-space tradeoff lower
bounds for learning problems to a wide array of natural scenarios where the 
sample space of tests is smaller than the sample space of inputs.  
Low degree polynomials with evaluation tests are just some of the natural
examples.
Our bound shows that if the 2-norm amplification curve for $M$ has the
required property, then to achieve learning success probability for
$M$ of at least $|A|^{-\varepsilon}$ for
some $\varepsilon>0$, either space
$\Omega(\log |A|\cdot \log |X|)$ or time $|A|^{\Omega(1)}$ is required.
This kind of bound is consistent even with what we know for very small sample
spaces of tests: for example, if $X$ is the space of linear functions over
$\mathbb{F}_2$ and $A$ is the standard basis $\set{e_1,\ldots,e_n}$ then,
even for exact identification, space $O(n)$ and time $O(n\log n)$ are 
necessary and sufficient by a simple coupon-collector analysis.

Thus far, we have assumed that the outcome of each random test in 
is one of two values.   We also sketch how to extend the approach to
multivalued outcomes. (We note that, though the mixing condition
of~\cite{DBLP:journals/eccc/MoshkovitzM17,MoshkovitzM17new} 
does not hold in the case of small sample spaces of tests,
\cite{DBLP:journals/eccc/MoshkovitzM17,MoshkovitzM17new} do apply in the case
of multivalued outcomes.)

Independent of the specific applications to learning from random examples
that we obtain, the measure of matrices that we introduce, the 2-norm
amplification curve on the positive orthant, seems likely to have signficant
applications in other contexts outside of learning.

\paragraph{Related work:}

Independently of our work, Garg, Raz, and Tal~\cite{grt:extractor-learn} have proven closely related results to ours.    The fundamental techniques are similarly grounded in the approach of~\cite{DBLP:journals/eccc/Raz17} though their method is based on viewing the matrices associated with learning problems as 2-source
extractors rather than on bounding the SDP relaxations of their 2-norm
amplification curves.     They use this for a variety of applications including
the polynomial learning problems we focus on here.

\subsection{Branching programs for learning}

Following Raz~\cite{DBLP:journals/eccc/Raz17},
we define the learning problem as follows.  
Given two non-empty sets, a set
$X$ of possible inputs, with a uniformly random prior distribution, and a set
$A$ of tests and a matrix $M: A\times X\rightarrow \set{-1,1}$, a learner tries
to learn an input $x\in X$ given a stream
of samples $(a^1,b^1),(a^2,b^2),\ldots$ where for every $t$, $a^t$ is chosen
uniformly at random from $A$ and $b^t=M(a^t,x)$.  
Throughout this paper we use the notation that $n=\log_2 |X|$ and
$m=\log_2 |A|$.

For example, parity learning is the special case of this learning problem where
$M(a,x)=(-1)^{a\cdot x}$.

Again following Raz~\cite{DBLP:conf/focs/Raz16}, the time and space of a
learner are modelled
simultaneously by expressing the learner's computation as a layered
branching program: a finite rooted directed acyclic multigraph with every
non-sink
node having outdegree $2|A|$, with one outedge for each $(a,b)$ with $a\in A$
and $b\in \set{-1,1}$ that leads to a node in the next layer.   
Each sink node $v$ is labelled by some $x'\in X$ which is the learner's guess
of the value of the input $x$.

The space $S$ used by the learning branching program is the $\log_2$ of the
maximum number of nodes in any layer and the time $T$ is the length of the
longest path from the root to a sink.

The samples given to the learner $(a^1,b^1),(a^2,b^2),\ldots$ based
on uniformly randomly chosen $a^1,a^2,\ldots \in A$ and an input $x\in X$
determines a (randomly chosen) {\em computation} path in the branching program.
When we consider computation paths we include the input $x$ in their
description.

The (expected) success probability of the learner is the probability for a
uniformly random $x\in X$ that on input $x$ a random computation path on input $x$
reaches a sink node $v$ with label $x'=x$.

\subsection{Progress towards identification}

Following \cite{DBLP:journals/eccc/MoshkovitzM17,DBLP:journals/eccc/Raz17}
we measure progress towards identifying $x\in X$ using
the ``expectation 2-norm'' over the uniform distribution:
For any set $S$, and $f:S\rightarrow \mathbb{R}$, define 
$$\|f\|_2=\left(\mathbb{E}_{s\in_R S} f^2(s)\right)^{1/2}=\left(\frac{1}{|S|}\sum_{s\in S} f^2(s)\right)^{1/2}.$$

Define $\Delta_X$ to be the space of probability distributions on $X$.
Consider the two extremes for the expectation 2-norm of elements of $\Delta_X$:
If $\mathbb{P}$ is the uniform distribution on $X$,
then $\|P\|_2 = 2^{-n}$.  This distribution represents the learner's knowledge
of the input $x$ at the start of the branching program.
On the other hand if $\mathbb{P}$ is point distribution on any $x'$, then
$\|P\|_2=2^{-n/2}$.

For each node $v$ in the branching program, there is an induced probability
distribution on $X$, $\mathbb{P}'_{x|v}$ which represents the distribution on
$X$ conditioned on the fact that the computation path passes through $v$.  
It represents the learner's knowledge of $x$ at the time that the computation
path has reached $v$.
Intuitively, the learner has made significant progress towards identifying the
input $x$ if $\|\mathbb{P}'_{x|v}\|_2$ is much larger than $2^{-n}$, say
$\|\mathbb{P}'_{x|v}\|_2\ge 2^{\delta n/2}\cdot 2^{-n}=2^{-(1-\delta/2)n}$.

The general idea will be to argue that for any fixed node $v$ in the branching
program that is at a layer $t$ that is $2^{o(m)}$,
the probability over a randomly chosen computation path that $v$ is the first
node on the path for which the learner has made significant progress is
$2^{-\Omega(mn)}$.  Since by assumption of correctness the learner makes
significant progress with at least $2^{-\varepsilon m}$ probability, there
must be at least $2^{\Omega(mn)}$
such nodes and hence the space must be $\Omega(mn)$.

Given that we want to consider the first vertex on a computation path at which
significant progress has been made it is natural to truncate a computation
path at $v$ if significant progress has been already been made at $v$ (and then
one should not count any path through $v$ towards the progress at some
subsequent node $w$).  Following~\cite{DBLP:journals/eccc/Raz17}, for
technical reasons we will also
truncate the computation path in other circumstances.  

\begin{defn}
\label{defn:truncation}
We define probability distributions $\mathbb{P}_{x|v}\in \Delta_X$ and the
$(\delta,\alpha,\gamma)$-truncation of the computation paths inductively as
follows:
\begin{itemize}
\item If $v$ is the root, then $\mathbb{P}_{x|v}$ is the uniform distribution on
$X$.
\item (Significant Progress) If $\|\mathbb{P}_{x|v}\|_2\ge 2^{-(1-\delta/2) n}$ then
truncate all computation paths at $v$.  We call vertex $v$ \emph{significant}
in this case.
\item (High Probability) Truncate the computation paths at $v$ for all inputs
$x'$ for which  $\mathbb{P}_{x|v}(x')\ge 2^{-\alpha n}$.  Let $\High(v)$ be the
set of such inputs.
\item (High Bias) Truncate any computation path at $v$ if it follows an
outedge $e$ of $v$ with label $(a,b)$ for which $|(M\cdot \mathbb{P}_{x|v})(a)|\ge 2^{-\gamma m}$. That is, we truncate the paths at $v$ if the outcome $b$ of
the next sample for $a\in A$ is too predictable in that it is highly
biased towards $-1$ or $1$ given the knowledge that the path was not truncated
previously and arrived at $v$.  
\item If $v$ is not the root then define $\mathbb{P}_{x|v}$ to be the 
conditional probability distribution on $x$ over all computation paths that
have not previously been truncated and arrive at $v$.
\end{itemize}
For an edge $e=(v,w)$ of the branching program, we also define a probability
distribution $\mathbb{P}_{x|e}\in \Delta_X$, which is the conditional
probability distribution on $X$ induced by the truncated computation paths
that pass through edge $e$.
\end{defn}

With this definition, it is no longer immediate from the assumption of
correctness that the truncated path reaches a significant node with
at least $2^{-\varepsilon m}$ probability.  However, we will see that a single
assumption about the
matrix $M$ will be sufficient to prove both that this holds and that 
the probability is $2^{-\Omega(nm)}$ that the path reaches any specific node
$v$ at which significant progress has been made.

\section{Norm amplification by matrices on the positive orthant}

By definition, for $\mathbb{P}\in \Delta_X$, 
$$\|M\cdot \mathbb{P}\|_2^2= \E_{a\in_R A} [|(M\cdot \mathbb{P})(a)|^2].$$
Observe that for $\mathbb{P}=\mathbb{P}_{x|v}$, the value
$|(M\cdot \mathbb{P}_{x|v})(a)|$ is precisely the expected bias of the answer
along a uniformly random outedge of $v$ (i.e., the advantage in predicting the
outcome of the randomly chosen test).

If we have not learned the input $x$, we would not expect to be able to predict
the outcome of a typical test; moreover, since any path that 
would follow a high bias test is truncated, it is essential to argue
that $\|M\cdot \mathbb{P}_{x|v}\|_2$ remains small at any node $v$ where there
has not been significant progress.

In~\cite{DBLP:journals/eccc/Raz17}, $\|M\cdot \mathbb{P}_{x|v}\|_2$ was
bounded using the matrix norm $\|M\|_2$ given by
$$\|M\|_2=\sup_{\substack{f:X\rightarrow \mathbb{R}\\f\ne 0}}\frac{\|M\cdot f\|_2}{\|f\|_2},$$
where the numerator is an expectation $2$-norm over $A$ and the denominator
is an expectation $2$-norm over $X$.
Thus $$\|M\|_2=\sqrt{\frac{|X|}{|A|}}\cdot \sigma_{\max}(M),$$
where $\sigma_{\max}(M)$ is the largest singular value of $M$ and $\sqrt{|X|/|A|}$ is a normalization factor.

In the case of the matrix $M$ associated with parity learning, $|A|=|X|$ and
all the singular values are equal to $\sqrt{|X|}$ so
$\|M\|_2=\sqrt{|X|}=2^{n/2}$.  With this bound, if $v$ is not a node of
significant progress then $\|\mathbb{P}_{x|v}\|_2\le 2^{-(1-\delta/2) n}$
and hence $\|M\cdot \mathbb{P}_{x|v}\|_2 \le 2^{-(1-\delta) n/2}$ which is
$1/|A|^{(1-\delta)/2}$ and hence small.

However, in the case of learning quadratic functions over $\mathbb{F}_2$,
the largest singular value of the matrix $M$ is still $\sqrt{|X|}$  (the uniform
distribution on $X$ is a singular vector) and so $\|M\|_2=|X|/\sqrt{|A|}$. 
But in that case, when $\|\mathbb{P}_{x|v}\|$ is $2^{-(1-\delta/2)n}$ we
conclude that $\|M\|_2 \cdot\|\mathbb{P}_{x|v}\|_2$ is at most
$2^{\delta n/2}/\sqrt{|A|}$ which is
much larger than 1 and hence a useless bound on $\|M\cdot \mathbb{P}_{x|v}\|_2$.

Indeed, the same kind of problem occurs in using the method
of~\cite{DBLP:journals/eccc/Raz17} for any learning
problem for which $|A|$ is $|X|^{o(1)}$:
If $v$ is a child of the root of the branching program at which the more
likely outcome $b$ of a single randomly chosen test $a\in A$ is remembered, then
$\|\mathbb{P}_{x|v}\|_2 \le \sqrt{2}/|X|$. 
However, in this case $|(M\cdot \mathbb{P}_{x|v})(a)|=1$ and so
$\|(M\cdot \mathbb{P}_{x|v})\|_2\ge |A|^{-1/2}$. 
It follows that $\|M\|_2\ge |X|/(2|A|)^{1/2}$ and when $|A|$ is $|X|^{o(1)}$
the derived upper bound on $\|M\cdot \mathbb{P}_{x|v'}\|_2$ at nodes $v'$
where $\|\mathbb{P}_{x|v'}\|_2\ge 1/|X|^{1-\delta/2}$ will be larger
than 1 and therefore useless.

We need a more precise way to bound $\|M\cdot \mathbb{P}\|_2$ as a function
of $\|\mathbb{P}\|_2$ than the single number $\|M\|_2$.   
To do this we will need to use the fact that $\mathbb{P}\in \Delta_X$ -- it has
a fixed $\ell_1$ norm and (more importantly) it is non-negative.

\begin{defn}
Let $M:X\times A\rightarrow \set{-1,1}$ be a $\pm 1$ matrix.
The 2-\emph{norm amplification curve} of $M$ is a map
$\tau_M:[0,1]\rightarrow \mathbb{R}$
given by
$$\tau_M(\delta)=\sup_{\substack{\mathbb{P}\in \Delta_X\\ \|\mathbb{P}\|_2\le 1/|X|^{1-\delta/2}}} \log_{|A|} (\|M\cdot \mathbb{P}\|_2).$$
\end{defn}

In other words, for $|X|=2^n$ and $|A|=2^m$, whenever $\|\mathbb{P}\|_2$ is at
most $2^{-(1-\delta/2) n}$, $\|M\cdot \mathbb{P}\|_2$ is at most
$2^{\tau_M(\delta)m}$.

\section{Theorems}
\label{sec:theorems}

Our lower bound for learning quadratic functions will be in two parts.  
First, we modify the argument of~\cite{DBLP:journals/eccc/Raz17} to use the
function $\tau_M$ instead of $\|M\|_2$:

\begin{theorem}
\label{thm:mainlb}
Let $M:X\times A\rightarrow \set{-1,1}$, $n=\log_2 |X|$, $m=\log_2 |A|$ and
assume that $m\le n$.
If $M$ has $\tau_M(\delta')<0$ for some fixed
constant $0<\delta'<1$, then there are constants $\varepsilon,\beta,\eta>0$
depending only on $\delta'$ and $\tau_M(\delta')$ such that any
algorithm that solves the learning problem for $M$ with
success probability at least $2^{-\varepsilon m}$ either requires
space at least $\eta nm$ or time at least $2^{\beta m}$.
\end{theorem}

(We could write the statement of the theorem to apply to all $m$ and $n$ by
replacing each occurrence of $m$ in the lower bounds with $\min(m,n)$.  When
$m\ge n$, we can use $\|M\|_2$ to bound $\tau_M(\delta')$ which yields
the bound given in~\cite{DBLP:journals/eccc/Raz17}.)

We then analyze the amplification properties of the matrix $M$ associated with
learning quadratic functions over $\mathbb{F}_2$.

\begin{theorem}
\label{thm:quadcurve}
Let $M$ be the matrix for learning (homogenous) quadratic functions over
$\mathbb{F}_2[z_1,\ldots, z_m]$.
Then $\tau_M(\delta)\le \frac{-(1-\delta)}8+\frac{5+\delta}{8m}$ for all 
$\delta\in [0,1]$.
\end{theorem}

The following corollary is then immediate.

\begin{corollary}
\label{cor:quadratic}
Let $m$ be a positive integer and $n=\binom{m+1}2$. For some $\varepsilon>0$,
any algorithm for learning quadratic functions over
$\mathbb{F}_2[z_1,\ldots, z_m]$ that succeeds with probability at least
$2^{-\varepsilon m}$ requires space $\Omega(mn)$ or time $2^{\Omega(m)}$.
\end{corollary}

This bound is tight since it matches the resources used by the learning
algorithms for quadratic functions given in the introduction up to constant
factors in the space bound and in the exponent of the time bound.

We obtain similar bounds for all low degree polynomials over $\mathbb{F}_2$.

\begin{theorem}
\label{thm:constant-d-curve}
Let $3\le d$ and $m\ge d^2$.
Let $M$ be the matrix for learning (homogenous) functions of degree at most
$d$ over $\mathbb{F}_2[z_1,\ldots, z_m]$. 
Then there is a constant $\lambda'_d>0$ depending on $d$ such that
$\tau_M(\delta)\le -\lambda'_d$
for all $0 <\delta <3/4$.
\end{theorem}

Again we have the following immediate corollary which is also asymptotically
optimal for constant degree polynomials.

\begin{corollary}
\label{cor:learn-constant-d}
Fix some integer $d\ge 2$. 
There is a $\varepsilon_d>0$ such that
for positive integers $m\ge d$ and $n=\sum_{i=1}^d \binom{m}i$, 
any algorithm for learning polynomial functions of degree at most $d$ over
$\mathbb{F}_2[z_1,\ldots, z_m]$ that succeeds with probability at least
$2^{-\varepsilon_d m}$ requires space $\Omega_d(mn)$ or time $2^{\Omega_d(m)}$.
\end{corollary}

For the case of learning larger degree polynomials where the $d$ can depend on 
the number of variables $m$, we can derive the following somewhat weaker
lower bound whose proof we only sketch.

\begin{theorem}
\label{thm:smalld}
There are constants $\zeta,\varepsilon >0$ such that for positive integer
$d\le (1-\zeta) \cdot m$ and $n=\sum_{i=1}^d \binom{m}i$.
any algorithm for learning polynomial functions of degree at most $d$ over
$\mathbb{F}_2[z_1,\ldots, z_m]$ that succeeds with probability at least
$2^{-\varepsilon m/d}$ requires space $\Omega(mn/d)$ or
time $2^{\Omega(m/d)}$.
\end{theorem}

We prove Theorem~\ref{thm:mainlb} in the next section. 
In Section~\ref{sec:sdp} we give a semidefinite programming relaxation of
that provides a strategy for bounding the norm amplification curve and in
Section~\ref{sec:polynomial} we give the applications of that method to the
matrices for learning low degree polynomials.
Finally, in Section~\ref{sec:multivalued}
we sketch how to extend the framework to learning problems for which
the tests have multivalued rather than simply binary outcomes.

\section{Lower Bounds over Small Sample Spaces}

In this section we prove Theorem~\ref{thm:mainlb}. 
Let $2/3<\delta'<1$ be the value given in the statement of the theorem,
To do this we define several positive constants that will be useful:
\begin{itemize}
\item $\delta=\delta'/6$,
\item $\alpha=1-2\delta$,
\item $\gamma=-\tau_M(\delta')/2$,
\item $\beta=\min(\gamma,\delta)/8$, and
\item $\varepsilon=\beta/2$.
\end{itemize}
Let $B$ be a learning branching program for $M$ 
with length at most $2^{\beta m}-1$ and
success probability at least $2^{-\varepsilon m}$.

We will prove that $B$ must have space $2^{\Omega(mn)}$.
We first apply the $(\delta,\alpha,\gamma)$-truncation procedure given in
Definition~\ref{defn:truncation} to yield $\mathbb{P}_{x|v}$ and
$\mathbb{P}_{e|v}$ for all vertices $v$ in $B$.

The following simple technical lemmas are analogues of ones proved
in~\cite{DBLP:journals/eccc/Raz17}, though we structure our argument somewhat
differently.  The first uses the bound on the
amplification curve of $M$ in place of its matrix norm.

\begin{lemma}
\label{lem:bias}
Suppose that vertex $v$ in $B$ is not significant.
Then $$\Pr_{a\in_R A}[|(M\cdot \mathbb{P}_{x|v})(a)|\ge 2^{-\gamma m}]\le 2^{-2\gamma m}.$$
\end{lemma}

\begin{proof}
Since $v$ is not significant $\|\mathbb{P}_{x|v}\|_2\le 2^{-(1-\delta/2)n}$.
By definition of $\tau_M$,
$$\E_{a\in_R A}[|(M\cdot \mathbb{P}_{x|v})(a)|^2]=\|M\cdot \mathbb{P}_{x|v}\|_2^2\le 2^{2\tau_M(\delta)m}\le 2^{2\tau_M(\delta') m}=2^{-4\gamma m}.$$
Therefore, by Markov's inequality, 
$$\Pr_{a\in_R A}[|(M\cdot \mathbb{P}_{x|v})(a)|\ge 2^{-\gamma m}]
=\Pr_{a\in_R A}[|(M\cdot \mathbb{P}_{x|v})(a)|^2\ge 2^{-2\gamma m}]
\le 2^{-2\gamma m}.$$
\end{proof}

\begin{lemma}
\label{lem:highprob}
Suppose that vertex $v$ in $B$ is not significant.
Then $$\Pr_{x'\sim\mathbb{P}_{x|v}}[x'\in \High(v)]\le 2^{-\delta n}.$$
\end{lemma}

\begin{proof}
Since $v$ is not significant,
$$\E_{x'\sim \mathbb{P}_{x|v}}[\mathbb{P}_{x|v}(x')]=\sum_{x'\in X}
(\mathbb{P}_{x|v}(x'))^2
=2^n\cdot \|\mathbb{P}_{x|v}\|^2_2\le 2^{-(1-\delta)n}=2^{-(alpha+\delta) n}.$$
Therefore, by Markov's inequality,
$$\Pr_{x'\sim \mathbb{P}_{x|v}}[x'\in \High(v)]=
\Pr_{x'\sim \mathbb{P}_{x|v}} [\mathbb{P}_{x|v}(x')\ge 2^{-\alpha n}]
\le 2^{-\delta n}.$$
\end{proof}

\begin{lemma}
\label{lem:significant}
The probability, over uniformly random $x'\in X$ and uniformly random
computation path $C$ in $B$ on input $x'$, that the truncated version $T$ of $C$
reaches a significant vertex of $B$ is at least $2^{-\beta m/2-1}$.
\end{lemma}

\begin{proof}
Let $x'$ be chosen uniformly at random from $X$ and consider the truncated
path $T$.
$T$ will not reach a significant vertex of $B$ only if one of the following
holds:
\begin{enumerate}
\item $T$ is truncated at a vertex $v$ where
$\mathbb{P}_{x|v}(x')\ge 2^{-\alpha n}$.
\item $T$ is truncated at a vertex $v$ because the next edge of $C$ is
labelled by $(a,b)$
where $|(M\cdot \mathbb{P}_{x|v})(a)|\ge 2^{-\gamma m}$.
\item $T$ ends at a leaf that is not significant.
\end{enumerate}
By Lemma~\ref{lem:highprob}, for each vertex $v$ on $C$,
conditioned on the truncated path reaching $v$, the probability that
$\mathbb{P}_{x|v}(x')\ge 2^{-\alpha n}$ is at most $2^{-\delta n}$. 
Similarly, by Lemma~\ref{lem:bias}, for each $v$ on the path,
conditioned on the truncated path reaching $v$, the probability that
$|(M\cdot \mathbb{P}_{x|v})(a)|\ge 2^{-\gamma m}$ is at most $2^{-2\gamma m}$.
Therefore, since $T$ has length at most $2^{\beta m}$, the probability that
$T$ is truncated at $v$ for either reason is at most $2^{\beta m}(2^{-2\gamma m}+2^{-\delta n})<2^{-\beta m+1}$ since $m\le n$ and
$\beta<\min(\gamma,\delta/2)$.

Finally, if $T$ reaches a leaf $v$ that is not significant then, conditioned
on arriving at $v$, the probability that the input $x'$ equals the label of
$v$ is at most $\max_{x''\in X} \mathbb{P}_{x|v}(x'')$.
Now 
$$\frac{\max_{x''\in X} \mathbb{P}_{x|v}(x'')}{2^{n/2}}\le 
\|\mathbb{P}_{x|v}\|_2< 2^{-(1-\delta/2) n}$$
since $v$ is not significant, so we have
$\max_{x''\in X} \mathbb{P}_{x|v}(x'')< 2^{-(1-\delta)n/2}= 2^{-(\alpha+\delta) n/2}$
and the 
probability that $B$ is correct conditioned on the truncated path reaching a
leaf vertex that is not significant is less than
$2^{-(\alpha+\delta) n/2}\le 2^{-\beta n}\le 2^{-\beta m}$ since $m\le n$.   

Since $B$ is correct with probability at least
$2^{-\varepsilon m}=2^{-\beta m/2}$ and these three
cases in which $T$ does not reach a significant vertex account for correctness
at most $3\cdot 2^{-\beta m}$, which is much less than half of $2^{-\beta m/2}$,
$T$ must reach a significant vertex with probability at least
$2^{-\beta m/2-1}$.
\end{proof}

The following lemma is the the key to the proof of the theorem. 

\begin{lemma}
\label{lem:space}
Let $s$ be any significant vertex of $B$.   There is an $\eta>0$ such that for
a uniformly random $x$ chosen
from $X$ and a uniformly random computation path $C$, the probability that
its truncation $T$ ends at $s$ is at most $2^{-\eta mn}$.
\end{lemma}

The proof of Lemma~\ref{lem:space} requires a delicate progress argument and
is deferred to the next subsection.   We first show how
Lemmas~\ref{lem:significant} and~\ref{lem:space} immediately imply
Theorem~\ref{thm:mainlb}.

\begin{proof}[Proof of Theorem~\ref{thm:mainlb}]
By Lemma~\ref{lem:significant}, for $x$ chosen uniformly at random from $X$
and $T$ the truncation of a uniformly random computation path on input $x$,
$T$ ends at a significant vertex with probability at least $2^{-\beta m/2-1}$.
On the other hand, by Lemma~\ref{lem:space}, for any significant vertex $s$,
the probability that $T$ ends at $s$ is at most $2^{-\eta mn}$.
Therefore the number of significant vertices must be at least
$2^{\eta mn - \beta m/2 -1}$ and since $B$ has length at most $2^{\beta m}$, 
there must be at least $2^{\eta mn- 3\beta m/2-1}$ significant vertices in
some layer.   Hence $B$ requires space $\Omega(mn)$.
\end{proof}

\subsection{Progress towards significance}

In this section we prove Lemma~\ref{lem:space} showing that for any particular
significant vertex $s$ a random truncated path reaches $s$ only with probability
$2^{-\Omega(mn)}$.   
For each vertex $v$ in $B$ let $\Pr[v]$ 
denote the probability over a random
input $x$, that the truncation of a random computation path in $B$ on input
$x$ visits $v$ and
for each edge $e$ in $B$ let $\Pr[e]$ 
denote the probability over a random
input $x$, that the truncation of a random computation path in $B$ on input
$x$ traverses $e$.

Since $B$ is a levelled branching program, the vertices of $B$ may
be divided into disjoint sets $V_t$ for $t=0,1,\ldots,T$ where $T$ is the
length of $B$ and $V_t$ is the set of vertices at distance $t$ from the root,
and disjoint sets of edges $E_t$ for $t=1,\ldots,T$ where $E_t$ consists
of the edges from $V_{t-1}$ to $V_t$.
For each vertex $v\in V_{t-1}$, note that by definition we only have
$$\Pr[v]\ge \sum_{(v,w)\in E_t} \Pr[(v,w)]$$
since some truncated paths may terminate at $v$.

For each $t$, since the truncated computation path visits at most one vertex
and at most one edge at level $t$, we obtain a sub-distribution on $V_t$ in
which the probability of $v\in V_t$ is $\Pr[v]$ and a corresponding
sub-distribution on $E_t$ in which the probability of $e\in E_t$ is $\Pr[e]$.
We write $v\sim V_t$ and $e\sim E_t$ to denote random selection from these
sub-distributions, where the outcome $\bot$ corresponds to the case that
no vertex (respectively no edge) is selected.

Fix some significant vertex $s$.
We consider the progress that a truncated path makes as it moves from the start
vertex to $s$. We measure the progress at a vertex $v$ as
$$\rho(v)=\frac{\langle \mathbb{P}_{x|v},\mathbb{P}_{x|s}\rangle}{\langle \mathbb{P}_{x|s},\mathbb{P}_{x|s}\rangle}.$$
Clearly $\rho(s)=1$.  We first see that $\rho$ starts out at a tiny value.

\begin{lemma}
\label{lem:start}
If $v_0$ is the start vertex of $B$ then $\rho(v_0)\le 2^{-\delta n}$.
\end{lemma}

\begin{proof}
By definition, $\mathbb{P}_{x|v_0}$ is the uniform distribution on $X$.
Therefore
$$\langle \mathbb{P}_{x|v_0},\mathbb{P}_{x|s}\rangle= \E_{x'\in X} [2^{-n}\cdot \mathbb{P}_{x|s}(x')]=2^{-2n}\cdot \sum_{x'\in X}\mathbb{P}_{x|v_0}(x')=2^{-2n}$$
since $\mathbb{P}_{x|s}$ is a probability distribution on $X$.
On the other hand, since $s$ is significant,
$\langle \mathbb{P}_{x|s},\mathbb{P}_{x|s}\rangle=\|\mathbb{P}_{x|s}\|_2^2\ge
2^{\delta n}\cdot 2^{-2n}$.
The lemma follows immediately.
\end{proof}

Since the truncated path is randomly chosen, the progress towards $s$ after
$t$ steps is a random variable.  Following~\cite{DBLP:journals/eccc/Raz17}, we show that 
not only is the increase in this expected value of this random variable in each
step very small, its higher moments also increase at a very small rate.  Define
$$\Phi_t=\E_{v\sim V_t}[(\rho(v))^{\gamma m}]$$
where we extend $\rho$ and define $\rho(\bot)=0$.
We will show that for $s\in V_t$, $\Phi_t$ is still $2^{-\Omega(mn)}$, which
will be sufficient to prove Lemma~\ref{lem:space}.

Therefore, Lemma~\ref{lem:space}, and hence Theorem~\ref{thm:mainlb}, will
follow from the following lemma.

\begin{lemma} 
\label{lem:progress}
For every $t$ with $1\le t\le 2^{\beta m}-1$,
$$\Phi_t \le \Phi_{t-1}\cdot(1+2^{-2\beta m}) +2^{-\gamma mn}.$$
\end{lemma}

\begin{proof}[Proof of Lemma~\ref{lem:space} from Lemma~\ref{lem:progress}]
By definition of $\Phi_t$ and Lemma~\ref{lem:start} we have
$\Phi_0\le 2^{-\delta \gamma mn}$.
By Lemma~\ref{lem:progress},  for every $t$ with $1\le t\le 2^{\beta m}-1$,
$$\Phi_t \le \sum_{j=0}^{t}  (1+2^{-2\beta m})^j\cdot 2^{-\delta \gamma mn}\\
< (t+1)\cdot (1+2^{-2\beta m})^t \cdot 2^{-\delta \gamma mn}.$$
In particular, for every $t\le 2^{\beta m}-1$,
$$\Phi_t\le 2^{\beta m}\cdot (1+2^{-2\beta m})^{2^{\beta m}}\cdot 2^{-\delta \gamma mn}\le
 e^{1/2^{\beta m}}\cdot 2^{-\delta \gamma mn+\beta m}.$$
Now fix $t^*$ to be the level of the significant node $s$.
Every truncated path that reaches $s$ will have contribution
$(\rho(s))^{\gamma m}=1$ times its probability of occurring to $\Phi_{t^*}$.
Therefore the truncation of a random computation path reaches $s$ with
probability at most $2^{-\eta mn}$ for $\eta=\delta\gamma/2$ and $m,n$
sufficiently large, which proves the lemma.
\end{proof}

We now focus on the proof of Lemma~\ref{lem:progress}.
Because $\Phi_t$ depends on the sub-distribution over $V_t$ and $\Phi_{t-1}$
depends on the sub-distribution over $V_{t-1}$, it is natural to consider
the analogous quantities based on the sub-distribution over the set $E_t$ of
edges that join $V_{t-1}$ and $V_t$.
We can extend the definition of $\rho$ to edges of $B$, where we write
$$\rho(e)=\frac{\langle \mathbb{P}_{x|e},\mathbb{P}_{x|s}\rangle}{\langle \mathbb{P}_{x|s},\mathbb{P}_{x|s}\rangle}.$$
Then define
$$\Phi'_t=\E_{e\sim E_t}[(\rho(e))^{\gamma m}].$$
Intuitively, there is no gain of information in moving from elements $E_t$ to
elements of $V_t$.  More precisely, we have the following lemma:

\begin{lemma}
\label{lem:edge-vertex}
For all $t$, $\Phi_t\le \Phi'_t$.
\end{lemma}

\begin{proof}
Note that for $v\in V_t$, 
since the truncated paths that follow some edge $(u,v)\in E_t$ are 
precisely those that reach $v$, by definition,
$\Pr[v]=\sum_{(u,v)\in E_t} \Pr[(u,v)]$.
Since the same applies separately to the set of truncated paths for each input
$x'\in X$ that reach $v$, for each $x'\in X$ we have
$$\Pr[v]\cdot \mathbb{P}_{x|v}(x') = \sum_{(u,v)\in E_t} \Pr[(u,v)]\cdot \mathbb{P}_{x|(u,v)}(x').$$
Therefore, 
$$\Pr[v]\cdot\frac{\langle \mathbb{P}_{x|v},\mathbb{P}_{x|s}\rangle}
{ \langle \mathbb{P}_{x|s},\mathbb{P}_{x|s}\rangle}
= \sum_{(u,v)\in E_t} \Pr[(u,v)]\cdot \frac{\langle \mathbb{P}_{x|(u,v)},\mathbb{P}_{x|s}\rangle}
{ \langle \mathbb{P}_{x|s},\mathbb{P}_{x|s}\rangle}
;$$
i.e., $\Pr[v]\cdot \rho(v) =\sum_{(u,v)\in E_t}\Pr[(u,v)]\cdot \rho((u,v))$.
Since $\Pr[v] =\sum_{(u,v)\in E_t}\Pr[(u,v)]$, 
by the convexity of the map $r\mapsto r^{\gamma m}$ we have
$$\Pr[v]\cdot (\rho(v))^{\gamma m}\le \sum_{(u,v)\in E_t}\Pr[(u,v)]\cdot (\rho((u,v))^{\gamma m}.$$
Therefore
$$\Phi_t=\sum_{v\in V_t} \Pr[v]\cdot (\rho(v))^{\gamma m} \le
\sum_{v\in V_t}\sum_{(u,v)\in E_t}\Pr[(u,v)]\cdot (\rho((u,v)))^{\gamma m}=
\sum_{e\in E_t}\Pr[e]\cdot (\rho(e))^{\gamma m}=\Phi'_t.$$
\end{proof}

Therefore, to prove Lemma~\ref{lem:progress} it suffices to prove that the
same statement holds with $\Phi_t$ replaced by $\Phi'_t$; that is,
$$\E_{e\in E_t} [(\rho(e))^{\gamma m}]\le (1+2^{-2\beta m})\cdot \E_{v\in V_{t-1}}[(\rho(v))^{\gamma m}]+2^{-\gamma mn}$$
$E_t$ is the disjoint union of the out-edges $\Gamma_{out}(v)$ for vertices
$v\in V_{t-1}$, so
it suffices to show that for each $v\in V_{t-1}$,
\begin{equation}
\label{growth}
\sum_{e\in \Gamma_{out}(v)} \Pr[e]\cdot (\rho(e))^{\gamma m}
\le (1+2^{-2\beta m})\cdot \Pr[v]\cdot (\rho(v))^{\gamma m}+2^{-\gamma mn}\cdot \Pr[v].
\end{equation}
Since any truncated path that follows $e$ must also visit $v$, we can write
$\Pr[e|v]=\Pr[e]/\Pr[v]$.
Moreover, both $\rho(v)$ and $\rho(e)$ have the same denominator
$\langle \mathbb{P}_{x|s},\mathbb{P}_{x|s}\rangle$ and
therefore, by definition,
inequality \eqref{growth},
and hence Lemma~\ref{lem:progress}, follows from the following lemma.

\begin{lemma}
\label{lem:vertex-edge}
For $v\in V_{t-1}$,
$$\sum_{e\in \Gamma_{out}(v)} \Pr[e|v]\cdot 
\langle \mathbb{P}_{x|e},\mathbb{P}_{x|s}\rangle^{\gamma m}
\le (1+2^{-2\beta m})\cdot 
\langle \mathbb{P}_{x|v},\mathbb{P}_{x|s}\rangle^{\gamma m} +2^{-\gamma mn}.$$
\end{lemma}

Before we prove Lemma~\ref{lem:vertex-edge},
following~\cite{DBLP:journals/eccc/Raz17}, we first prove
two technical lemmas, the first relating the distributions for
$v\in V_{t-1}$ and edges $e\in E_t$ and the second upper bounding
$\|\mathbb{P}_{x|s}\|_2$.

\begin{lemma}
\label{lem:technical-v-e}
Suppose that $v\in V_{t-1}$ is not significant and $e=(v,w)\in E_t$ has
$\Pr[e]>0$ and label $(a,b)$.
Then for $x'\in X$, $\mathbb{P}_{x|e}(x')>0$ only if $x'\notin \High(v)$ and
$M(a,x')=b$, in which case
$$\mathbb{P}_{x|e}(x')= c_e^{-1}\cdot \mathbb{P}_{x|v}(x')$$
where $c_e\ge \frac{1}{2} - 2^{-\gamma m-1}-2^{-\delta n}$.
\end{lemma}

\begin{proof}
If $|(M\cdot \mathbb{P}_{x|v})(a)|\ge 2^{-\gamma m}$ then by
definition of truncation we also will have $\Pr[e]=0$.
Therefore, since $\Pr[e]>0$, $e$ is not a high bias edge -- that is,
$|(M\cdot \mathbb{P}_{x|v})(a)|< 2^{-\gamma m}$ -- and hence
$$\Pr_{x'\sim \mathbb{P}_{x|v}}[M(a,x')=b]> \frac{1}{2}(1-2^{-\gamma m}).$$
Let $\mathcal{E}_e(x')$ be the event that both $M(a,x')=b$ and $x'\notin \High(v)$
and define 
$$c_e=\Pr_{x'\sim \mathbb{P}_{x|v}}[\mathcal{E}_e(x')].$$
If $\mathcal{E}_e(x')$ fails to hold for all $x'$, i.e.,  $x'\in \High(v)$ or $M(a,x')\ne b$,
then any truncated path on input $x'$ that reaches $v$ will not continue along
$e$ and hence $\Pr[e]=0$.
On the other hand, since $\Pr[e]>0$, if $\mathcal{E}_e(x')$ holds for some $x'$ then any
truncated path on input $x'$ that reaches $v$ will continue precisely if the
test chosen at $v$ is $a$, which happens with probability $2^{-m}$ for each
such $x'$.  The total probability over $x'\in X$, conditioned that the
truncated path on $x'$ reaches $v$, that the path continues along $e$ is then
$2^{-m}\cdot c_e$.
Therefore, if $x'\in \mathcal{E}_e$ then
$\mathbb{P}_{x|e}(x')=\frac{2^{-m}\cdot\mathbb{P}_{x|v}(x')}{2^{-m}\cdot c_e}
=c_e^{-1}\cdot \mathbb{P}_{x|v}(x')$.
Now by Lemma~\ref{lem:highprob},
$$\Pr_{x'\sim \mathbb{P}_{x|v}}[x'\in \High(v)]\le 2^{-\delta n}$$
and so 
$$c_e=\Pr_{x'\sim \mathbb{P}_{x|v}}[M(a,x')=b\mbox{ and }x'\notin \High(v)]> \frac{1}{2}-2^{-\gamma m-1}-2^{-\delta n}$$
as required.
\end{proof}

We use this lemma together with an argument similar to that of
Lemma~\ref{lem:edge-vertex} to upper bound $\|\mathbb{P}_{x|s}\|_2$ for our
significant vertex $s$.

\begin{lemma}
\label{lem:s-bound}
$\|\mathbb{P}_{x|s}\|_2\le 4\cdot 2^{-(1-\delta/2)n}$.
\end{lemma}

\begin{proof}
The main observation is that $s$ is the first significant vertex of any
truncated path that reaches it and so the probability distributions of each of
the immediate predecessors $v$ of $s$ must have bounded expectation $2$-norm
and, by Lemma~\ref{lem:technical-v-e} and the proof idea from
Lemma~\ref{lem:edge-vertex}, the $2$-norm of the distribution at $s$
cannot grow too much larger than those at its immediate predecessors.

By Lemma~\ref{lem:technical-v-e}, if $e=(v,s)$ and $\Pr[e]>0$, then
$$\|\mathbb{P}_{x|e}\|_2\le c_e^{-1}\cdot \|\mathbb{P}_{x|v}\|\le
c_e^{-1} 2^{-(1-\delta/2)n}\le 4\cdot 2^{-(1-\delta/2)n}$$
since $v$ is not significant and 
$c_e\ge \frac{1}{2} - 2^{-\gamma m-1}-2^{-\delta n}> \frac{1}{4}$ for $m$
(and hence $n$) sufficiently large.
Let $\Gamma_{in}(s)$ be the set of edges $(v,s)$ in $B$.
$\Pr[s]=\sum_{e=(v,s)\in \Gamma_{in}(s)} \Pr[e]$ and for each
$x'\in X$,
$$\Pr[s]\cdot \mathbb{P}_{x|s}(x') = \sum_{e=(v,s)\in \Gamma_{in}(s)} \Pr[e]\cdot \mathbb{P}_{x|e}(x').$$
Since 
$\Pr[s]=\sum_{e=(v,s)\in \Gamma_{in}(s)} \Pr[e]$, by convexity of the
map $r\mapsto r^2$, we have
$$\Pr[s]\cdot (\mathbb{P}_{x|s}(x'))^2 = \sum_{e=(v,s)\in \Gamma_{in}(s)} \Pr[e]\cdot (\mathbb{P}_{x|e}(x'))^2.$$
Summing over $x'\in X$ we have 
$$\Pr[s]\cdot \|\mathbb{P}_{x|s}\|_2^2\le 
\sum_{e=(v,s)\in \Gamma_{in}(s)} \Pr[e]\cdot \|\mathbb{P}_{x|e}\|_2^2
\le \sum_{e=(v,s)\in \Gamma_{in}(s)} \Pr[e]\cdot (4\cdot 2^{-(1-\delta/2)n})^2
= \Pr[s]\cdot (4\cdot 2^{-(1-\delta/2)n})^2.$$
Therefore $\|\mathbb{P}_{x|s}\|\le 4\cdot 2^{-(1-\delta/2)n}$ as
required since $\Pr[s]>0$.
\end{proof}

To complete the proof of Lemma~\ref{lem:progress}, and hence
Lemma~\ref{lem:space}, it only remains to prove Lemma~\ref{lem:vertex-edge}.

\subsubsection{Proof of Lemma~\ref{lem:vertex-edge}}
Since we know that if $v\in V_{t-1}$ is significant then any edge
$e\in \Gamma_{out}(v)$ has $\Pr[e]=0$, we can assume without loss of
generality that $v$ is not significant.   

Define $g:X\rightarrow \mathbb{R}$ by 
$$g(x')=\mathbb{P}_{x|v}(x')\cdot \mathbb{P}_{x|s}(x')$$ and note that
$\langle \mathbb{P}_{x|v},\mathbb{P}_{x|s} \rangle=\E_{x'\in X}[g(x')]$.
For $x'\in X$ define
$$f(x')=\begin{cases}g(x')&x'\notin \High(v)\\0&\mbox{otherwise}\end{cases}$$
and  let $F=\sum_{x'\in X} f(x')$.  
For every edge $e$ where  $\langle \mathbb{P}_{x|e},\mathbb{P}_{x|s}\rangle>0$, we have $F>0$.

\begin{sloppypar}
The function $f$ induces a new probability distribution on $X$,
$\mathbb{P}_f$, given by
$\mathbb{P}_f(x')=f(x')/\sum_{x\in X} f(x)=f(x')/F$ in which each point
$x'\in X\setminus\High(v)$ is
chosen with probability proportional to its contribution to
$\langle \mathbb{P}_{x|v},\mathbb{P}_{x|s}\rangle$ and each $x'\in \High(v)$
has probability 0.
\end{sloppypar}

\medskip
{\sc Claim:} Let $(a,b)$ be the label on an edge $e$, then $$\langle \mathbb{P}_{x|e},\mathbb{P}_{x|s}\rangle
\le (2c_e)^{-1} (1+|(M\cdot \mathbb{P}_f)(a)|)\cdot F/2^n
\le (2c_e)^{-1} (1+|(M\cdot \mathbb{P}_f)(a)|)\cdot \langle \mathbb{P}_{x|v},\mathbb{P}_{x|s}\rangle$$ where $c_e$ is given by Lemma~\ref{lem:technical-v-e}.

We first prove the claim.
By Lemma~\ref{lem:technical-v-e} and the definition of $f$, 
$$\mathbb{P}_{x|e}(x')\cdot \mathbb{P}_{x|s}(x')=\begin{cases}c_e^{-1} \cdot f(x')&\mbox{if }M(a,x')=b\\ 0&\mbox{otherwise.}\end{cases}$$
Therefore
\begin{align*}
\langle \mathbb{P}_{x|e},\mathbb{P}_{x|s}\rangle
&=\E_{x'\in_R X}[ \mathbb{P}_{x|e}(x')\cdot \mathbb{P}_{x|s}(x')]\\
&=\E_{x'\in_R X}[ c_e^{-1} f(x')\cdot \mathbf{1}_{M(a,x')=b}]\\
&=\E_{x'\in_R X}[ c_e^{-1} f(x')\cdot (1+b\cdot M(a,x'))/2]\\
&= (2c_e)^{-1}\cdot \left (\E_{x'\in_R X}[ f(x')] +b\cdot \E_{x'\in_R X}[ M(a,x')\cdot f(x')]\right)\\
&\le (2c_e)^{-1}\cdot \left (\E_{x'\in_R X}[ f(x')] +\left|\E_{x'\in_R X}[ M(a,x')\cdot f(x')]\right|\right)\\
&= (2c_e)^{-1}\cdot 2^{-n}\cdot F\cdot \left(1+\frac{\left|\E_{x'\in_R X}[ M(a,x')\cdot f(x')]\right|}{F}]\right)\\
&= (2c_e)^{-1}\cdot 2^{-n}\cdot F\cdot (1+|(M\cdot \mathbb{P}_f)(a)|)\\
&\leq (2c_e)^{-1}\cdot (1+|(M\cdot \mathbb{P}_f)(a)|)\cdot 
\langle \mathbb{P}_{x|v},\mathbb{P}_{x|s}\rangle
\end{align*}
since $2^{-n}\cdot F=\E_{x'\in_R X}[f(x')]\le \E_{x'\in_R X}[g(x')]=
\langle \mathbb{P}_{x|v},\mathbb{P}_{x|s}\rangle$, which proves the claim.

By Lemma~\ref{lem:technical-v-e}, $2c_e\ge 1-2^{-\gamma m}-2^{1-\delta n}$
and so $(2c_e)^{-1}\le 1+2^{-\sigma m}\le 2$ for $\sigma=\min(\gamma,\delta)/2$
since $m\le n$ for $m$ sufficiently large.
We consider two cases:

\medskip
\noindent
{\sc Case $F\le 2^{-n}$:}  In this case, since $\mathbb{P}_f$ is a probability
distribution, for every $a$, $|(M\cdot \mathbb{P}_f)(a)|\le \max_{x'\in X} |M(a,x')|=1$
and from the claim we obtain for every edge $e\in \Gamma_{out}(v)$,
$\langle \mathbb{P}_{x|e},\mathbb{P}_{x|s}\rangle\le 2\cdot (2c_e)^{-1} \cdot 2^{-2n}$.
Therefore 
$\sum_{e\in \Gamma_{out}(v)} \Pr[e|v]\cdot 
\langle \mathbb{P}_{x|e},\mathbb{P}_{x|s}\rangle^{\gamma m}$
is at most  $[4\cdot 2^{-2n}]^{\gamma m}\le 2^{-\gamma mn}$
for $n\ge 2$.

\medskip
\noindent
{\sc Case $F\ge 2^{-n}$:}  In this case we will show that $\|\mathbb{P}_f\|_2$
is not too large and use this together with the bound on the 2-norm
amplification curve of $M$
to show that $\|M\cdot \mathbb{P}_f\|_2$ is small.  This will be important
because of the following connection:

By the Claim, we have
\begin{equation}
\sum_{e\in \Gamma_{out}(v)} \Pr[e|v]\cdot 
\langle \mathbb{P}_{x|e},\mathbb{P}_{x|s}\rangle^{\gamma m}
\le \sum_{e\in \Gamma_{out}(v)} \Pr[e|v] 
[(2c_e)^{-1} (1+|(M\cdot \mathbb{P}_f)(a_e)|)]^{\gamma m}\cdot \langle \mathbb{P}_{x|v},\mathbb{P}_{x|s}\rangle^{\gamma m} \label{eq:key}
\end{equation}
where $a_e$ is the test labelling edge $e$.
By definition, for each $a\in A$ there are precisely two edges
$e,e'\in \Gamma_{out}(v)$ with $a_{e}=a_{e'}=a$ and
$\Pr[e|v]+\Pr[e'|v]\le 1/|A|$ since the next test is chosen uniformly at random
from $A$. 
(It would be equality but some tests $a$ have high bias and in that
case $\Pr[e|v]=\Pr[e'|v]=0$.)
Previously, we also observed that $(2c_e)^{-1}\le 1+2^{-\sigma m}$ where
$\sigma=\min(\gamma,\delta)/2$.
Therefore,
\begin{align*}
\sum_{e\in \Gamma_{out}(v)} \Pr[e|v]\cdot 
\langle \mathbb{P}_{x|e},\mathbb{P}_{x|s}\rangle^{\gamma m}
&\le \sum_{a\in A} \frac{1}{|A|}
[(1+2^{-\sigma m})\cdot (1+|(M\cdot \mathbb{P}_f)(a)|)]^{\gamma m}\cdot \langle \mathbb{P}_{x|v},\mathbb{P}_{x|s}\rangle^{\gamma m}\\
&= (1+2^{-\sigma m})^{\gamma m}\cdot 
\E_{a\in_R A} [(1+|(M\cdot \mathbb{P}_f)(a)|)^{\gamma m}]\cdot 
\langle \mathbb{P}_{x|v},\mathbb{P}_{x|s}\rangle^{\gamma m}.
\end{align*}
To prove the lemma we therefore need to bound
$\E_{a\in_R A} [(1+|(M\cdot \mathbb{P}_f)(a)|)^{\gamma m}]$.
We will bound this by first analyzing $\|M\cdot \mathbb{P}_f\|_2$.

By definition, 
$$\|f\|^2_2=\E_{x'\in_R X} \1_{x'\notin \High(v)}\cdot \mathbb{P}^2_{x|v}(x')\cdot
\mathbb{P}^2_{x|s}(x')
\le 2^{-2\alpha n}\cdot \E_{x'\in_R X} \mathbb{P}^2_{x|s}(x') =
2^{-2\alpha n}\cdot \|\mathbb{P}_{x|s}\|^2_2.$$
Therefore, by Lemma~\ref{lem:s-bound}, and the fact that $F\geq 2^{-n}$,
$$\|\mathbb{P}_f\|_2=\frac{\|f\|_2}{F}\le \frac{2^{-\alpha n}\cdot \|\mathbb{P}_{x|s}\|_2}{2^{-n}}
\le 2^{(1-\alpha) n}\cdot 4\cdot 2^{-(1-\delta/2)n}= 2^{(1-\alpha+\delta/2+2/n)n}\cdot 2^{-n}.$$
Since, for sufficiently large $n$, 
$$1-\alpha+\delta/2+2/n = 2\delta + \delta/2 +2/n \leq 3\delta = \delta'/2,$$  
we have $\|\mathbb{P}_f\|_2 \leq 2^{-(1-\delta'/2)n}$. So, $\|M\cdot \mathbb{P}_f\|_2\le 2^{\tau_M(\delta')m}=2^{-2\gamma m}$.
Thus $\E_{a\in_R A}[|(M\cdot \mathbb{P}_{f})(a)|^2]=\|M\cdot \mathbb{P}_{f}\|_2^2\le 2^{-4\gamma m}$. 
So, by Markov's inequality, 
$$\Pr_{a\in_R A}[|(M\cdot \mathbb{P}_{f})(a)|\ge 2^{-\gamma m}]
=\Pr_{a\in_R A}[|(M\cdot \mathbb{P}_{f})(a)|^2\ge 2^{-2\gamma m}]
\le 2^{-2\gamma m}.$$
Therefore, since we always have $|(M\cdot \mathbb{P}_f)(a)|\le 1$, 
\begin{align*}
\E_{a\in_R A} [(1+|(M\cdot \mathbb{P}_f)(a)|)^{\gamma m}]
&\le \E_{a\in_R A} [\1_{|(M\cdot \mathbb{P}_f)(a)|\le 2^{-\gamma m}}\cdot (1+2^{-\gamma m})^{\gamma m}]
+\E_{a\in_R A} [\1_{|(M\cdot \mathbb{P}_f)(a)>2^{-\gamma m}} \cdot
2^{\gamma m}]\\
&\le (1+2^{-\gamma m})^{\gamma m}+ 2^{-2\gamma m}\cdot 2^{\gamma m}\\
&\le (1+2^{-\gamma m})^{\gamma m}+ 2^{-\gamma m}\\
&\le 1+2^{-\gamma m/2}
\end{align*}
for $m$ sufficiently large.
Therefore, 
the total factor increase over $\langle \mathbb{P}_{x|v},\mathbb{P}_{x|s}\rangle^{\gamma m}$ is at most
$(1+2^{-\sigma m})^{\gamma m}\cdot (1+2^{-\gamma m/2})$ where
$\sigma=\min(\gamma,\delta)/2$.
Therefore, for sufficiently large $m$ this is at most
$1+2^{-\min(\gamma,\delta)m/4}$.
Since $\beta\le \min(\gamma,\delta)/8$ this is at most $1+2^{-2\beta m}$ as required
to prove Lemma~\ref{lem:vertex-edge}.

\section{An SDP Relaxation for Norm Amplification on the Positive Orthant}
\label{sec:sdp}

For a matrix $M$, 
$$\tau_M(\delta)=\sup_{\substack{\mathbb{P}\in \Delta_X\\ \|\mathbb{P}\|_2\le 1/|X|^{1-\delta/2}}} \log_{|A|} (\|M\cdot \mathbb{P}\|_2).$$
That is, $\tau_M(\delta)=\frac12 \log_{|A|} OPT_{M,\delta}$ where $OPT_{M,\delta}$ is the
optimum of the following quadratic program:
\begin{equation}
\begin{aligned}\label{original}
\text{Maximize}  &\quad \|M\cdot \mathbb{P}\|_2^2=\langle M\cdot \mathbb{P},M\cdot \mathbb{P}\rangle,& \\
\text{subject to:}&&\\
&\quad \sum_{i\in X} \mathbb{P}_i=1,&\\
&\quad \sum_{i\in X} \mathbb{P}_i^2\le |X|^{\delta-1},&\\
&\quad \mathbb{P}_i\geq 0 &\quad \text{for all }i\in X.
\end{aligned}
\end{equation}
Instead of attempting to solve \eqref{original}, presumably a difficult
quadratic program, we consider the following semidefinite program (SDP):
\begin{equation} \label{SDP-relaxation}
\begin{aligned}
\text{Maximize}&\quad \langle M^T M,U\rangle &\\
\text{subject to:} &&\\
[V]&\quad U\succeq 0,&\\
[w]&\quad \sum_{i,j\in X} U_{ij}=1,  &\\
[z]&\quad \sum_{i\in X} U_{ii}\le |X|^{\delta-1},&\\
&\quad U_{ij}\geq 0 &\quad\text{for all } i,j\in X.
\end{aligned}
\end{equation}
Note that for any $\mathbb{P}\in \Delta_X$ achieving the optimum value
of \eqref{original} the positive semidefinite matrix
$U=\mathbb{P}\cdot \mathbb{P}^T$ 
has the same value in \eqref{SDP-relaxation}, and hence \eqref{SDP-relaxation}
is an SDP relaxation of \eqref{original}.
In order to upper bound the value of \eqref{SDP-relaxation},
we consider its dual program:
\begin{equation} \label{dual-SDP}
\begin{aligned}
\text{Minimize}  &\quad w+z\cdot |X|^{\delta-1} &\\
\text{subject to:}&&\\
[U]&\quad V\succeq 0,&\\
[U_{ii}] &\quad w+z\geq V_{ii}+(M^T M)_{ii}/2^m,&\quad\text{for all }i\in X\\
[U_{ij}] &\quad w\geq V_{ij}+(M^T M)_{ij}/2^m,&\quad\text{for all }i\ne j \in X\\
&\quad z\ge 0
\end{aligned}
\end{equation}
or equivalently,
\begin{equation} \label{dual-SDP1}
\begin{aligned}
\text{Minimize}  &\quad w+z\cdot |X|^{\delta}\cdot |X|^{-1} &\\
\text{subject to:}&&\\
&\quad V\succeq 0,&\\
&\quad z I + w J\geq V+ M^T M/2^m,\\
&\quad z\ge 0.
\end{aligned}
\end{equation}
where $I$ is the identity matrix and $J$ is the all 1's matrix over $X\times X$.

Any dual solution of \eqref{dual-SDP1} yields an upper bound on the optimum
of \eqref{SDP-relaxation} and hence $OPT_{M,\delta}$ and $\tau_M(\delta)$.
To simplify the complexity of analysis we restrict ourselves to considering
semidefinite matrices $V$ that are suitably chosen Laplacian matrices.
For any set $S$ in $X\times X$ and any $\alpha:S\rightarrow \mathbb{R}_+$
the Laplacian matrix associated with $S$ and $\alpha$
is defined by $L_{(S,\alpha)}:=\sum_{(i,j)\in S}\alpha(i,j) L_{ij}$ where
$L_{ij}=(e_i-e_j)(e_i-e_j)^T$ for the standard basis $\{e_i\}_{i\in X}$ .
Intuitively, in the dual SDP \eqref{dual-SDP1}, by adding matrix $V=L_{S,\alpha}$ for
suitable $S$ and $\alpha$ depending on $M$ we can shift weight from the
off-diagonal
entries of $M^T M$ to the diagonal where they can be covered by the
$z+w$ entries on the diagonal rather than being covered by the $w$ values in
the off-diagonal entries.  This will be advantageous for us since the
objective function has much smaller coefficient for $z$ which helps cover
the diagonal entries than coefficient for $w$, which is all that covers the
off-diagonal entries.

\begin{defn}
Suppose that $N\in \mathbb{R}^{X\times X}$ is a symmetric matrix.
For $\kappa\in \mathbb{R}_+$, define
$W_\kappa(N)=\max_{i\in X} \sum_{j\in X:\ N_{i,j}> \kappa} (N_{i,j}-\kappa)$.
\end{defn}

The following lemma is the basis for our bounds on $\tau_M(\delta)$.

\begin{lemma}
\label{SDP-lemma}
Let $\kappa\in \mathbb{R}_+$.  Then
\begin{equation*}
OPT_{M,\delta}\le (\kappa+W_\kappa(M^T M)\cdot |X|^{\delta-1})/2^m.
\end{equation*}
\end{lemma}

\begin{proof}
Let $N=M^T\times M$.  For each off-diagonal entry of $N$ with $N(i,j)>\kappa$,
include matrix $L_{ij}$ with coefficient $(N(i,j)-\kappa)/2^m$ in the sum for
the Laplacian $V$.
By construction, the matrix $V+M^T M/2^m$ has off-diagonal entries
at most $\kappa/2^m$ and diagonal entries at most
$(\kappa+W_\kappa(M^T M))/2^m$.  The solution to \eqref{dual-SDP1}
with $w=\kappa/2^m$ and $z=W_\kappa(M^T M)/2^m$ is therefore feasible,
which yields the bound as required.
\end{proof}

For specific matrices $M$, we will obtain the required bounds on
$\tau_M(\delta)<0$ for some $0<\delta<1$ by showing that we can
set $\kappa=|A|^\gamma$ for some $\gamma<1$ and obtain that
$W_\kappa(M^T M)$ is at most
$\kappa\cdot |X|^{\gamma'}$ for some $\gamma'<1$.

\section{Applications to Low Degree Polynomial Functions}
\label{sec:polynomial}

\subsection{Quadratic Functions over $\mathbb{F}_2$}

In this section we prove Theorem~\ref{thm:quadcurve} on the norm amplification
curve of the matrix $M$ associated with learning homogeneous quadratic
functions over $\mathbb{F}_2$.
(Over $\mathbb{F}_2$, $x_i^2=x_i$, so being homogeneous is equivalent in a
functional sense to having no constant term.)
Let $m\in \mathbb{N}$ and $n=\binom{m+1}{2}$.
The learning matrix $M:\{0,1\}^m\times \{0,1\}^n\rightarrow \{-1,1\}$ for
quadratic functions is the partial Hadamard matrix given by
$M(a,x)=(-1)^{\sum_{1\leq i\leq j\leq m}x_{ij}a_ia_j}$.
We show that it has the following properties:

\begin{proposition}\label{main_counting}
Let $M$ be the matrix for learning homogeneous quadratic functions over 
$\mathbb{F}_2[z_1,\ldots,z_m]$ and let $N=M^T\cdot M$. 
\begin{enumerate}
\item Every row of $N_x$ for $x\in X$ contains the same multi-set of values.
\item $N_{xx}=2^m$ and $N_{xy}\in \{\pm 2^{m-1},\pm 2^{m-2},\cdots,\pm 2^{\lceil \frac m 2\rceil},0\}$ for $x\neq y\in \set{0,1}^n$.
\item For $i>0$, let $c_i$ be the number of entries equal to $2^{m-i}$ in each
row of $N$;
then 
\[c_i=(2^{2i-1}+2^{i-1})\frac{\prod_{j=0}^{2i-1}(2^m-2^j)}{\prod_{j=1}^i2^{2j-1}(2^{2j}-1)}\leq 2^{2im}\]
\end{enumerate}
\end{proposition}

Given Proposition~\ref{main_counting} we can derive
Theorem~\ref{thm:quadcurve}.

\begin{proof}[Proof of Theorem~\ref{thm:quadcurve}]
Let the threshold $\kappa=2^{m-k}$ for some integer $k$ to be determined later.
By Lemma~\ref{SDP-lemma} with $X=\{0,1\}^n$, for \eqref{original}, we have
$OPT_{M,\delta}\le (\kappa+W_\kappa(N) 2^{(\delta-1)n})/2^m$ where
$N=M^T\cdot M$.
By definition of $W_\kappa$ and Proposition~\ref{main_counting}, for any
$x\in X$ we have 
\begin{equation*}
W_\kappa(N)
\le \sum_{y\in X: N_{xy}> 2^{m-k}} N_{xy}
=\sum_{t=0}^{k-1}  c_t\cdot 2^{m-t} 
\le \sum_{t=0}^{k-1} 2^{2tm}\cdot 2^{m-t}
=\sum_{t=0}^{k-1} 2^{(2m-1)t+m}
< 2^{(2m-1)k}.
\end{equation*}
Thus for any $k$,
$$OPT_{M,\delta}\le (2^{m-k}+2^{2m-1)k+(\delta-1)n})/2^m=2^{-k}+2^{(2m-1)k-(1-\delta)m(m+1)/2-m}.$$
The first term is larger for $k\le (1-\delta) m/4 + (3-\delta)/4$ so to balance
them as much as possible we choose
$k=\lfloor (1-\delta) m/4 + (3-\delta)/4\rfloor\ge (1-\delta)m/4 -(1+\delta)/4$.
Hence $OPT_{M,\delta}\le 2\cdot 2^{-k} \le 2^{-\frac {1-\delta}4 m +\frac{5+\delta}4}$
Therefore,
$\tau_M(\delta)=\frac12 \log_{2^m} OPT_{M,\delta}\le -\frac{(1-\delta)}8+\frac{(5+\delta)}{8m}$ as required. 
\end{proof}

The proof of Proposition~\ref{main_counting} is in the appendix but in the
next section we outline the connection to the weight distribution of
Reed-Muller codes over $\mathbb{F}_2$, and show how bounds on the weight
distribution of such codes allow us to derive time-space tradeoffs for
learning larger degree $\mathbb{F}_2$ polynomials as well.

\newcommand{\w}{\textrm{weight}}
\subsection{Connection to the Weight Distibution of Reed-Muller Codes}

It remains to prove Proposition~\ref{main_counting} and the bounds for larger
degrees.
Let $d\ge 2$ be an integer.  For any integer $m\ge d$, for the learning
problem for (homogenous) $\mathbb{F}_2$ polynomials of degree at most
$d$, we have $n=\sum_{i=1}^d \binom{m}{d}$.
Recall that $N=M^T\cdot M$.
Let $M_x$ denote the $x$-th column of $M$ where $x\in \set{0,1}^n$.
Then $N_{xy}=2^m \cdot \langle M_x,M_y\rangle$.
Recall that for $a\in \{0,1\}^m$ and $x\in \{0,1\}^n$,
$x(a)=\sum_{S: 1\leq |S|\leq d} x_S \prod_{i\in S} a_i$ over $\mathbb{F}_2$.

\begin{proposition}
Let $\mathbf{0}=0^n$.
Then $\langle M_x,M_y\rangle=\langle M_{\mathbf{0}},M_{x+y}\rangle$.
\end{proposition}

\begin{proof}
\begin{align*}
\langle M_x,M_y\rangle&=\E_{a\in \{0,1\}^m} M_x(a)M_y(a)
=\E_{a\in \{0,1\}^m} (-1)^{x(a)}(-1)^{y(a)}
=\E_{a\in \{0,1\}^m} (-1)^{x(a)+y(a)}\\
&=\E_{a\in \{0,1\}^m} (-1)^{(x+y)(a)}
=\E_{a\in \{0,1\}^m} M_{\mathbf{0}}(a)M_{x+y}(a)
=\langle M_{\mathbf{0}},M_{x+y}\rangle
\end{align*}
\end{proof}
Since the mapping $y\mapsto x+y$ for $x\in \set{0,1}^n$ is
1-1 on $\set{0,1}^n$, this immediately implies part 1 of
Proposition~\ref{main_counting}.

Therefore, we only need to examine a fixed row $N_\mathbf{0}$ of $N$,
where each entry
$$N_{\mathbf{0}x}=\sum_{a\in \set{0,1}^m} M(a,x)=\sum_{a\in \set{0,1}^m}(-1)^{x(a)}.$$
For $x\in X$, define $\w(x)=|\set{a\in \set{0,1}^m\ :\ x(a)=1}|$.
By definition, for $x\in \set{0,1}^n$, 
$N_{\mathbf{0}x}=\sum_{a\in \set{0,1}^m}(-1)^{x(a)}=2^m-2\cdot \w(x)$.
Thus, understanding the function $W_\kappa(N)$ that we use to derive our
bounds, reduces to understanding the distribution of $\w(x)$ for
$x\in \set{0,1}^n$.
In particular, our goal of showing that for some $\kappa$ for which
$(\kappa+W_\kappa(N))/2^m$ is at most $2^{2\tau m}$ for some $\tau<0$ 
follows by showing that the distribution of $\w(x)$ is tightly concentrated
around $2^{m}/2$.

We can express this question in terms of (a small variation of)
the Reed-Muller error-correcting code $RM(d,n)$ (see, e.g.\cite{berlekamp-book}).

\begin{defn}
The Reed-Muller code $RM(d,m)$ over $\mathbb{F}_2$ is the set of vectors
$\set{G\cdot x\mid x\in \set{0,1}^n}$ where $G$ is the $2^m\times n$ matrix
for $n=\sum_{t=0}^d \binom{m}{t}$ over $\mathbb{F}_2$
with rows indexed by vectors $a\in \set{0,1}^m$ and columns indexed by
subsets $S\subseteq [m]$ with $|S|\le d$ given by $G(a,S)=\prod_{i\in S} a_i$.
\end{defn}

Evaluating $\w(x)$ for all $x\in \set{0,1}^n$ is almost exactly that of
understanding the distribution of Hamming weights of the vectors in $RM(m,d)$,
a question with a long history.

Because we assumed that the constant term of our polynomials is 0 (in order to
view the learning problem as a variant of the parity learning problem with a
smaller set of test vectors), we need to make a small change in the
Reed-Muller code.
Consider the subcode $RM'(d,m)$ of $RM(d,m)$ having the 
generator matrix $G'$ that is the same as $G$ but with the (all 1's) column
indexed by $\emptyset$ removed.
In this case, for each $x\in \set{0,1}^m$, by definition, $\w(x)$ is 
precisely the Hamming weight of the vector $G'\cdot x$ in $RM'(d,m)$.

Now by definition
$$RM(d,m)=\set{y\mid y\in RM'(d,m)\mbox{ or }\overline{y}\in RM'(d,m)}$$
where $\overline{y}$ is the same as $y$ with every bit flipped.
In particular, $RM'(d,m)\subset RM(d,m)$ and the distribution of the
weights for $RM(d,m)$
is symmetric about $2^{m-1}$, whereas the distribution of weights in
$RM'(d,m)$ is not necessarily symmetric.

For the special case that $d=2$, Sloane and
Berlekamp~\cite{DBLP:journals/tit/SloaneB70} derived an exact
enumeration of the number of vectors of each weight in $RM(2,m)$.

\begin{proposition}\cite{DBLP:journals/tit/SloaneB70}
\label{lem:RM(2,m)-count}
The weight of every codeword of $RM(2,m)$ is of the form $2^{m-1}\pm 2^{m-i}$
for some integer $i$ with $1\le i\le \lceil m/2\rceil$ or precisely $2^{m-1}$
and the number of codewords of weight $2^{m-1}+2^{m-i}$ or $2^{m-1}-2^{m-i}$ 
is precisely
$$2^{i(i+1)} \prod_{j=0}^{i-1} \frac{2^{m-2j}(2^{m-2j-1}-1)}{2^{2(j+1)}-1}.$$
\end{proposition}

The exact enumeration for the values of the weight function $\w$ for $RM'(2,m)$
that corresponds to the values in
Proposition~\ref{main_counting} are only slightly different from those for
$RM(2,m)$ in Proposition~\ref{lem:RM(2,m)-count} and 
Proposition~\ref{lem:RM(2,m)-count} gives the same asymptotic 
bound we used in the proof of Theorem~\ref{thm:quadcurve} since words of weight
$2^{m-1}-2^{m-i}$ correspond to entries of value $2^{m-i+1}$ in the matrix $N$.
We give a proof of the exact counts of
Proposition~\ref{main_counting} in the appendix.

The minimum distance, the smallest 
weight of a non-zero codeword, in $RM(d,m)$ is $2^{m-d}$ but
for $2<d<m-2$, no exact enumeration of the weight distribution of the code
$RM(d,m)$ is known.
It was a longstanding problem even to approximate the
number of codewords of different weights in $RM(d,m)$.  
Relatively recently, bounds on
these weights that are good enough for our purposes were shown by
Kaufman, Lovett, and Porat~\cite{DBLP:journals/tit/KaufmanLP12}.

\begin{proposition}\cite{DBLP:journals/tit/KaufmanLP12}
\label{prop:klp}
Let $d\le m$ be a positive integer and $1\le k\le d-1$.  There is a constant
$C_d>0$ such that for $0\le \varepsilon\le 1/2$ and
$2^{m-d}\le W_{k,\varepsilon} = 2^{m-k}(1-\varepsilon)$, the number
of codewords of weight at most $W_{k,\varepsilon}$ in $RM(d,m)$ is at most 
$$(1/\varepsilon)^{C_d m^{d-k}}.$$
\end{proposition}

\begin{corollary}
\label{cor:constant-d}
For $2\le d$ and $m\ge d^2$ there is a constant $\lambda_d>0$ such that if $M$
is the $2^m\times 2^n$ matrix associated with learning homogenous polynomials of
degree at most $d$ over $\mathbb{F}_2$ and $N=M^T\cdot M$
then for $\kappa= 2^{(1-\lambda_d)m}$
we have $W_\kappa(N)\le 2^{n/4}$.
\end{corollary}

\begin{proof}
Let $C_d$ be the constant from Proposition~\ref{prop:klp}, and
$\lambda_d= \frac{1}{20 C_d d!}$.   Then applying Proposition~\ref{prop:klp}
with $k=1$ and $\varepsilon=2^{-\lambda_d m}$, we obtain that the number of
words of weight at most $W_{1,\varepsilon}=2^{m-1}-2^{(1-\lambda_d)m-1}$ in
code $RM(d,m)$ is at most 
\begin{equation*}
(1/\varepsilon)^{C_d m^{d-k}}=2^{\lambda_d C_d m^d}=2^{\frac{m^d}{20 d!}}.
\end{equation*}
Now for $d\le\sqrt{m}$, $m(m-1)\dots(m-d+1)\ge m^d/e$, and hence
$\frac{m^d}{20d!}<\frac{e}{20} \sum_{i=1}^d \binom{m}{i}=\frac{e}{20}\cdot n$.
Therefore there are at most  $2^{en/20}$ codewords of weight at most
$2^{m-1}-2^{(1-\lambda_d)m-1}$ in $RM'(d,m)$.  As we have shown,
the entries of value at least
$\kappa=2^{(1-\lambda_d)m}$ in each of the rows (the first row) of $N$
precisely correspond to these codewords.  The total weight of these matrix
entries in a row is at most $2^m\cdot 2^{en/20}\le 2^{n/4}$ since $m\le n/10$
for $2\le d\le \sqrt{m}$ and hence $W_\kappa(N)\le 2^{n/4}$ as required.
\end{proof}

We now have the last tool we need to prove Theorem~\ref{thm:constant-d-curve}.

\begin{proof}[Proof of Theorem~\ref{thm:constant-d-curve}]
Let $0<\delta\le 3/4$.
Let $M$ be the $2^m\times 2^n$ matrix associated with learning homogenous
polynomials of degree at most $d$ over $\mathbb{F}_2$, let $N=M^T\cdot M$ and
let $d$, $\lambda_d$, and $\kappa$ satisfy the properties of
Corollary~\ref{cor:constant-d}.
By Lemma~\ref{SDP-lemma} with $X=\set{0,1}^n$, we
have 
$$OPT_{M,\delta}\le (\kappa+W_\kappa(N)\cdot 2^{(\delta-1)n})/2^m
\le (2^{(1-\lambda_d)m}+1)/2^m\le 2^{-\lambda_d m +1}.$$
Therefore $\tau_M(\delta)\le -\frac{\lambda_d}2 + \frac{1}{2m}$ which
yields a $\lambda'_d$ as required.
\end{proof}

A closer examination of the proof of the statement of Proposition~\ref{prop:klp}
in~\cite{DBLP:journals/tit/KaufmanLP12} shows that the
following more precise statement is also true:

\begin{proposition}
\label{prop:klp-stronger}
Let $d\le m$ be a positive integer and $1\le k\le d-1$.  
For $0\le \varepsilon\le 1/2$ and
$2^{m-d}\le W_{k,\varepsilon} = 2^{m-k}(1-\varepsilon)$, the number
of codewords of weight at most $W_{k,\varepsilon}$ in $RM(d,m)$ is at most 
$$2^{c(d^2+d\log_2(1/\varepsilon)) \sum_{i=0}^{d-k} \binom{m}{i}}$$
for some absolute constant $c>0$.
\end{proposition}

This allows us to sketch the proof of a weaker form of Theorem~\ref{thm:smalld}
with $1/d^2$ instead of $1/d$ throughout.

\begin{proof}[Sketch of Proof of Weaker Form of Theorem~\ref{thm:smalld}]
By applying the same method as in the proof of Corollary~\ref{cor:constant-d}
to Proposition~\ref{prop:klp-stronger} with $k=1$. 
Since $n\ge\frac{m}{d}\cdot  \sum_{i=0}^{d-1}\binom{m}{i}$,
in order to obtain a bound that $W_\kappa(N)\le 2^{n/4}$,
it suffices to have $c(d^2+d\log_2(1/\varepsilon))\le \frac{m}{10d}$.
In particular, there is a sufficiently small $\zeta>0$ such that for
$d\le \zeta m^{1/3}$ we can choose $\varepsilon=2^{-m/(20cd^2)}$ and 
$\kappa=2^{(1-1/(20cd^2))m}$.

Applying Lemma~\ref{SDP-lemma} we obtain that for $\delta\in (0,3/4)$ 
if $M$ is the learning matrix for polynomials of degree at most $d$ over
$\mathbb{F}_2$ then we have $OPT_{M,\delta}\le 2^{-c'm/d^2}$ for some 
constant $c'>0$ and hence $\tau_M(\delta)\le -c"/d^2$ for some $c">0$ and
$0<\delta <3/4$.

Now we cannot apply Theorem~\ref{thm:mainlb} as it is, because $\tau_M(\delta)$
is not bounded away from 0 by a constant independent of $d$ since $d$ may
grow with $m$.
However, if we examine the proof of Theorem~\ref{thm:mainlb} we observe that
it still goes through with $\alpha$ and $\delta$ constant and with
$\gamma,\beta,\varepsilon,\sigma>0$ all of the form $\Theta(1/d^2)$ which
imply that $\eta$ is $\Theta(1/d^2)$.
\end{proof}

To obtain Theorem~\ref{thm:smalld} we use the following result proven by
Ben-Eliezer, Hod, and Lovett~\cite{DBLP:journals/cc/Ben-EliezerHL12}.

\begin{proposition}
\label{prop:behl-bias}
For $\varepsilon>0$ there are constants $c_1, c_2$ with $0<c_1,c_2<1$ such that if $p$ is a uniformly random degree $d$ polynomial over
$\mathbb{F}^m_2$ and $d\le (1-\varepsilon)m$ then 
$$\Pr[|\E_{a\in \set{0,1}^m} (-1)^{p(a)}|>2^{-c_1 m/d}]\le 2^{-c_2 \sum_{i=0}^d \binom{m}{i}}.$$
\end{proposition}

From this form we can obtain the bound fairly directly.

\begin{proof}[Sketch of Proof of Theorem~\ref{thm:smalld}]
Fix $\varepsilon>0$ and let $0<c_1, c_2<1$ be the constants depending on
$\varepsilon$ from Proposition~\ref{prop:behl-bias}.
Let $\delta=c_2/2$ so $0<\delta<1/2$.
Let $M$ be the $2^m\times 2^n$ matrix associated with learning homogenous
polynomials of degree at most $d$ over $\mathbb{F}_2$, let $N=M^T\cdot M$ and
Setting $\kappa=2^{(1-c_1/d)m}$, by Proposition~\ref{prop:behl-bias} at
most $2^{(1-c_2)(n+1)}$ 
polynomials $p$ have entries $N_{0p}$ larger than $\kappa$.
Each such entry has value at most $2^m$ so 
$W_\kappa(N)\le 2^m\cdot 2^{(1-c_2)(n+1)}$.
by Lemma~\ref{SDP-lemma} with 
$X=\set{0,1}^n$ we have
$$OPT_{M,\delta}\le (\kappa+W_\kappa(N)\cdot 2^{(\delta-1)n})/2^m
\le 2^{-c_1 m/d}+2^{(\delta - c_2)n+1}\le 2^{-c_1 m/d}+2^{1-\delta n}$$
which is at most $2^{-c' m/d}$ for some constant $c'>0$. 
Hence $\tau_M(\delta)\le c'/d$.

Now we cannot apply Theorem~\ref{thm:mainlb} as it is, because $\tau_M(\delta)$
is not bounded away from 0 by a constant independent of $d$ since $d$ may
grow with $m$.
However, if we examine the proof of Theorem~\ref{thm:mainlb} we observe that
it still goes through with $\alpha$ and $\delta$ constant and with
$\gamma,\beta,\varepsilon,\sigma>0$ all of the form $\Theta(1/d)$ which
imply that $\eta$ is $\Theta(1/d)$.
\end{proof}

\section{Multivalued Outcomes from Tests}
\label{sec:multivalued}

We now extend the definitions to tests that can produce one of $r$ different
values rather than just 2 values.   
In this case, the learning problem can be expressed by learning matrix $M:A\times X\rightarrow \set{\omega^j\mid j\in \set{0,1,\ldots, r-1}}$ where
$\omega=e^{2\pi i/r}$ is a primitive $r$-th root of unity and each node in the
learning branching program has $r$ outedges associated with each possible
$a\in A$, one for each of the outcomes $b$.
Thus $M\in \mathbb{C}^{A\times X}$ and $M\cdot \mathbb{P}$ is a complex vector.
In this case the lower bound argument follows along very similarly to the
proof of Theorem~\ref{thm:mainlb} with a few necessary changes that we briefly
outline here:

As usual, for $z\in \mathbb{C}$, we replace absolute value by $|z|$ given
by $|z|^2=\overline z\cdot z$ where $\overline z$ is the complex conjugate of $z$.
For a vector $v\in \mathbb{C}^A$, by definition
$\|v\|^2_2=\langle v^*, v\rangle=\frac{1}{|A|} v^*\cdot v$, where
$v^*$ is the conjugate transpose of $v$. 
Using this, we can define the matrix norm $\|M\|_2$ as well as the 
2-norm amplification curve $\tau_M(\delta)$ for $\delta\in [0,1]$ as before.

Using these extended definitions, the notions of the truncated paths,
significant vertices, and high bias values are the same as before.

In the generalization of Lemma~\ref{lem:technical-v-e}, the $1/2$ in defining
$c_e$ will be replaced by $1/r$ so that $rc_e$ is very close to 1.   Also, in
the proof of Lemma~\ref{lem:vertex-edge}, the indicator function
$\mathbf{1}_{M(a,x')=b}$ is no longer equal to $1+b\cdot M(a,x')/2$ but rather
equal to
$$[1+\overline b\cdot M(a,x') + (\overline b \cdot M(a,x'))^2+\cdots (\overline b \cdot M(a,x'))^{r-1}]/r,$$
whose expectation we can bound in a similar way using the norm amplification
curve for $M$ since the powers simply rotate
these values on the unit circle.

\paragraph{Applications to learning polynomials}
The application to the degree 1 case of learning linear functions over $\mathbb{F}_p$ for prime $p$ follows directly since the amplification curve can be
bounded by the matrix norm.

For low degree polynomials over $\mathbb{F}_{2^t}$ we can also obtain a similar
lower bound from the $\mathbb{F}_2$ case, though the lower bounds do not grow
with $t$.

More generally, we can consider extending our results for $\mathbb{F}_2$
to the case of learning low degree polynomials over prime
fields $\mathbb{F}_p$ for $p>2$.   (It is natural to define $m$ to be the
number of variables $\log_p |A|$ and $n=\log_p |X|$ in this case rather than
$\log_2 |A|$ and $\log_2 |X|$ respectively.)
After applying the semidefinite relaxation in a similar way (using $M^*M$
instead of $M^T M$), we can reduce the lower bound problem to understanding
the distribution of the norms of values in each row of $M^* M$.   By
similar arguments this reduces to understanding the distributions of
$\sum_{a\in \mathbb{F}^m_p} \omega^{p(a)}$ where $\omega=e^{2\pi i/p}$.
Unfortunately, the natural extension of the bounds
of~\cite{DBLP:journals/tit/KaufmanLP12} to
Reed-Muller codes over $\mathbb{F}_p$, as shown in~\cite{bl:reed-muller-small}
are not sufficient here.  
We would need that all but a $p^{-\Theta(n)}$ fraction of
polynomials $p$ have $|\sum_{a\in \mathbb{F}^m_p} \omega^{p(a)}|$ very small
which means that almost all codewords are at most $p^{(1-\Omega(1))m}$
out of  balance between the $p$ values.  By symmetry, in the case that
$p=2$ this is equivalent to showing that only a $2^{-\Theta(n)}$ 
fraction of codewords lie
at distance at most $\frac12(1-\varepsilon)$ of the all 0's codeword for 
$\varepsilon=2^{\Omega(m)}$.
Indeed, \cite{DBLP:journals/cc/Ben-EliezerHL12} showed sharper bounds for
the deviation from balance.
In the case of larger $p$, \cite{bl:reed-muller-small} show that few codewords lie
within a $\frac{p-1}p (1-\varepsilon)$ distance of the all 0's codeword but
this is no longer enough to yield the balance we need and their bound does
not allow $\varepsilon$ to be as small as in the analysis for $p=2$
in~\cite{DBLP:journals/tit/KaufmanLP12} or
\cite{DBLP:journals/cc/Ben-EliezerHL12}.   
Bhowmick and Lovett~\cite{DBLP:journals/corr/0001L15} also analyze the
conditions under which small deviations of the sort we wish to bound 
occur, but do not provide a bound on the fraction of such occurrences.

\newpage
\bibliographystyle{plain}
\bibliography{mybib}
\newpage
\appendix
\newcommand{\Q}[1]{\mathbb{F}_2^{2}[z_1,\cdots,z_{#1}]}
\newcommand{\LL}[1]{\mathbb{F}_2^{1}[z_1,\cdots,z_{#1}]}
\newcommand{\val}{\mathbf{val}}
\newcommand{\orbvec}{\mathbf{type}}
\section{Proof of Proposition~\ref{main_counting}}

In this section we prove Proposition~\ref{main_counting}.
We already showed part 1 in Section~\ref{sec:polynomial} so
we only need to study the first row of $N$,
where each entry 
$$N_{\mathbf{0}x}=\sum_{a\in \{0,1\}^m} M(a,x)=\sum_{a\in \set{0,1}^m} (-1)^{x(a)}$$
and $x(a)=\sum_{i\le j} x_{ij} a_i a_j$.

Part 2 was essentially shown by Kasami~\cite{kasami1966weight} and generalized
by Sloane and Berlekamp~\cite{DBLP:journals/tit/SloaneB70} to give the precise
the weight distribution of $RM(2,m)$, which can be used to derive Part 3 also.
We give a direct proof of both parts using a lemma of
Dickson~\cite{dickson:book} characterizing the structure 
of quadratic forms over $\mathbb{F}_{2^n}$ that seems to be related to the
argument  used by McEliece~\cite{mceliece1967linear} to give an alternative
proof of Sloane and Berlekamp's result.

\begin{lemma}[Dickson's Lemma~\cite{dickson:book}]
For every quadratic form $q$ with coefficients in $\mathbb{F}_{2^t}$ in
variables $z_1,\ldots, z_m$,
there is an invertible linear transformation $T$ over $\mathbb{F}_{2^t}$ such
that for $z'=T\cdot z$,
there is a unique $k\le m/2$ such that precisely one of the following holds:
\begin{align*}
q&\equiv z'_1\cdot z'_2+z'_3\cdot z'_4+\cdots + z'_{2k-1} z'_{2k}+(z'_{2k+1})^2\\
\noalign{or}\\
q&\equiv z'_1\cdot z'_2+z'_3\cdot z'_4+\cdots + 
z'_{2k-1} z'_{2k}+\lambda (z'_{2k-1})^2+(z'_{2k})^2
\end{align*}
where $\lambda=0$ or $z'_{2k-1} z'_{2k}+\lambda (z'_{2k-1})^2+(z'_{2k})^2$ is 
irreducible in $\mathbb{F}_{2^t}$.
\end{lemma}

In the case that $t=1$, the squared terms in Dickson's Lemma become linear 
terms so the degree 2 parts can be assumed to be
$z'_1\cdot z'_2+z'_3\cdot z'_4+\cdots + z'_{2k-1} z'_{2k}$ for some $k\le m/2$.

\begin{defn}
Write $\mathcal{Q}_m:=\Q{m}$ to denote the set of all pure quadratic forms
$q$ over $\mathbb{F}_{2}$ with $m$ variables given by $\sum_{i<j} q_{ij} z_i z_j$.
Write $\mathcal{L}_m:=\LL{m}$ to denote the set of all linear polynomials $\ell$
over $\mathbb{F}_{2^t}$ with $m$ variables given by $\sum_i \ell_i z_i$.
\end{defn}

We can write every homogeneous quadratic polynomial $p$ over $\mathbb{F}$
uniquely as $q+\ell$ for $q\in \mathcal{Q}_m$ and $\ell\in \mathcal{L}_m$.
For any $p=q+\ell$ write 
$$\val(p)=(-1)^{p(a)}=(-1)^{\sum_{a\in \set{0,1}^m} p_{ij} a_i a_j}.$$
We will study the distribution of $\val(p)$ over all quadratic polynomials $p$
by partitioning the set of polynomials based on their purely quadratic part $q$.

For $q\in \mathcal{Q}_m$ we say that the \emph{type} of $q$, $\orbvec(q)$,
is the multiset of $2^m$ values given by $\val(q+\ell)$ over all
$\ell\in \mathcal{L}_m$.  We represent this as type as a set of pairs
$(v:j)$ where $j=\#\set{\ell\in \mathcal{L}_m\mid \val(q+\ell)=v}$.

\begin{lemma}
\label{lem:basic-orbit}
If $q=z_1 z_2+\ldots +z_{2k-1}z_{2k}$ then 
$\orbvec(q)=\set{(2^{m-k}\ :\ 2^{2k-1}+2^{k-1}),(0\ :\ 2^m-2^{2k}), (-2^{m-k}\ :\ 2^{2k-1}-2^{k-1})}$.
\end{lemma}

\begin{proof}
First observe that if the linear term $\ell$ has a non-zero coefficient of 
any $z_j$ for $j>2k$ then $\val(q+\ell)=0$.   Therefore there are only $2^{2k}$
of the $2^m$ linear functions $\ell$ such that $\val(q+\ell)$ can be non-zero.

Consider first the case that $k=1$.   $z_1 z_2$ is 1 on precisely 1/4 of the
inputs and so $\val(z_1 z_2)=2^m(3/4-1/4)=2^{m-1}$. 
Observe that $z_1 z_2 + z_1=z_1(z_2+1)$ and
$z_1z_2+z_2=(z_1+1)z_2$ will be equivalent under a 1-1 mapping of the space
of inputs  and also have $\val=2^{m-1}$.
Finally observe that $z_1 z_2 + z_1+z_2=(z_1+1)(z_2+1)+1$ and hence
$\val(z_1 z_2 + z_1+z_2)=-2^{m-1}$ since the final $+1$ flips the signs after
a 1-1 mapping of the inputs. 
This yields 3 values of $2^{m-1}$ and 1 value of $-2^{m-1}$ as required.

For larger $k$, observe that $\val/2^m$ is fractional discrepancy on $\set{0,1}^m$ which is therefore multiplicative over the sums of independent functions.
Therefore $\val(q)=2^{m-k}$.  
Furthermore, for $\ell$ supported on $\set{z_1,\ldots,z_{2k}}$,
we have $\val(q+\ell)=2^{m-k}$ if there are an even number of $i$ such that
$\ell$ contains $z_{2i-1}+z_{2i}$ and $\val(q+\ell)=-2^{m-k}$ if there are an
odd number of such $i$.  Since there is 3 choices per value of $i$ where
$\ell$ does not contain $z_{2i-1}+z_{2i}$ for every choice that does, we
can write the number of $\ell$ such that $\val(q+\ell)=2^{m-k}$ minus the number
of $\ell$ such that $\val(q+\ell)=-2^{m-k}$ as
$\sum_{i=0}^k (-1)^i 3^{k-1}$ which equals $2^k$.
This yields the claim.
\end{proof}

The following lemma follows immediately from Lemma~\ref{lem:basic-orbit}
and Dickson's Lemma.

\begin{lemma}
\label{orbit-type}
For every $q\in\mathcal{Q}_m$ there is some integer $k$ with $0\le k\le m/2$ such
that
$\orbvec(q)=\set{(2^{m-k}\ :\ 2^{2k-1}+2^{k-1}),(0\ :\ 2^m-2^{2k}), (-2^{m-k}\ :\ 2^{2k-1}-2^{k-1})}$.
\end{lemma}

\begin{proof}
By Dickson's Lemma, there is an invertible linear transformation $T$ that
maps any element $q\in \mathcal{Q}_m$ to a polynomial whose quadratic part $q'$
is of the form $z'_1 z'_2+\cdots+ z'_{2k-1}z'_{2k}$. Since the transformation
$T$ is invertible it preserves the $\val$ function and hence it preserves
$\orbvec$.  Using Lemma~\ref{lem:basic-orbit} we obtained the claimed result.
\end{proof}

This proves the Part 2 of Proposition~\ref{main_counting}.
Then the following lemma will complete the proof of Proposition~\ref{main_counting}.

\begin{lemma}
For $0\le i\le m/2$, 
let $c_i(m)$ denote the number of $q\in \mathcal{Q}_m$ such that
$\orbvec(q)=\set{(2^{m-i}\ :\ 2^{2i-1}+2^{i-1}),(0\ :\ 2^m-2^{2i}), (-2^{m-i}\ :\ 2^{2i-1}-2^{i-1})}$.
Then
\[
c_i(m)=\frac{\prod_{j=0}^{2i-1}(2^m-2^j)}{\prod_{j=1}^i2^{2j-1}(2^{2j}-1)}
\]
\end{lemma}

\begin{proof}
By induction.  When $m=0$, there is only one type, and $c_0(0)=1$.

Assume we that we have proved the claim for $m$.
Let us consider the case for $m+1$.
Lemma~\ref{orbit-type} says all $p=q+\ell$ for $q\in \mathcal{Q}_{m+1}$ and
$\ell\in \mathcal{L}_{m+1}$
such that $|\val(p)|=2^{m+1-i}$ have
$\orbvec(q)=\set{(2^{m+1-i}\ :\ 2^{2i-1}+2^{i-1}),(0\ :\ 2^{m+1}-2^{2i}), (-2^{m+1-i}\ :\ 2^{2i-1}-2^{i-1})}$.
So, the total count of quadratic polynomials $p$ on $m+1$ variables
with $|\val(p)|=2^{m+1-i}$ is precisely $2^{2i}\cdot c_{i}(m+1)$.

For $q\in \mathcal{Q}_{m+1}$ we can uniquely write $q=q'+\ell' x_{m+1}$
where $q'\in \mathcal{Q}_m$ and $\ell'\in \mathcal{L}_m$ and write $\ell=\ell''+b x_{m+1}$ where $\ell''\in \mathcal{L}_m$ and $b\in \set{0,1}$.
We can determine $\val(p)$ for $p=q+\ell$ by splitting the space of assignments into two equal parts depending on the value assigned to $x_{m+1}$.
Therefore 
$$\val(p)=\val(q+\ell)=\val(q'+\ell'')+(-1)^b\cdot \val(q'+\ell'+\ell'').$$
Thus, by the inductive hypothesis, the only way that this can yield
$|\val(p)|=2^{m+1-i}$ is if one of the following 3 cases holds:

\medskip\noindent{\sc Case 1},
$|\val(q'+\ell'')|=2^{m-i}$ and $\val(q'+\ell'+\ell'')=(-1)^b \val(q'+\ell'')$:\\
There are precisely $c_i(m)$ choices of such a $q'$ and for each $q$ there
$2^{2i}$ choices of an $\ell''$ such that $|\val(q'+\ell'')|=2^{m-i}$. 
For each such choice there
will be $2^{2i}$ choices of $\ell'$ such that $|\val(q'+\ell'+\ell')|=2^{m-i}$
and then only one choice of $b$ that will yield equal signs so that
$|\val(p)|=2^{m+1-i}$.

\medskip\noindent{\sc Case 2},
$|\val(q'+\ell'')|=2^{m+1-i}$ and $\val(q'+\ell'+\ell'')=0$:\\
There are precisely $c_{i-1}(m)$ choices of $q'$ and $2^{2(i-1)}$ choices
of $\ell''$ so that $|\val(q'+\ell'')|=2^{m+1-i}$.   For each such choice
there will be $2^m-2^{2(i-1)}$ choices of $\ell'$ such that
$\val(q'+\ell'+\ell'')=0$.

\medskip\noindent{\sc Case 3},
$\val(q'+\ell'')=0$ and $|\val(q'+\ell'+\ell'')|=2^{m+1-i}$:\\
This is symmetrical to the previous case and has the same number of choices.

\noindent
So, by induction hypothesis,
we have
\[
2^{2i}\cdot c_{i}(m+1)=2^{2i}\cdot 2^{2i}\cdot c_i(m)+2\cdot 2\cdot 2^{2(i-1)}\cdot (2^m-2^{2(i-1)})\cdot c_{i-1}(m)
\]
Now we can plug in the induction hypothesis to get
\begin{align*}
c_{i}(m+1)&=2^{2i}\cdot c_i(m)+ (2^m-2^{2(i-1)})\cdot c_{i-1}(m)\\
&=2^{2i}\cdot\frac{\prod_{j=0}^{2i-1}(2^m-2^j)}{\prod_{j=1}^i2^{2j-1}(2^{2j}-1)}+ (2^m-2^{2(i-1)})\cdot\frac{\prod_{j=0}^{2i-3}(2^m-2^j)}{\prod_{j=1}^{i-1}2^{2j-1}(2^{2j}-1)}\\
&=\frac{\prod_{j=0}^{2i-3}(2^m-2^j)}{\prod_{j=1}^i2^{2j-1}(2^{2j}-1)}(2^{2i}(2^m-2^{2i-2})(2^m-2^{2i-1})+(2^{m}-2^{2i-2})2^{2i-1}(2^{2i}-1))\\
&=\frac{\prod_{j=0}^{2i-3}(2^m-2^j)}{\prod_{j=1}^i2^{2j-1}(2^{2j}-1)}(2^{2i-1}(2^m-2^{2i-2})(2^{m+1}-1))\\
&=\frac{\prod_{j=0}^{2i-1}(2^{m+1}-2^j)}{\prod_{j=1}^i2^{2j-1}(2^{2j}-1)}
\end{align*}
as required.
\end{proof}

\end{document}